\pdfoutput=1

\documentclass[11pt]{article}
\usepackage{acl}
\usepackage{float}

\usepackage{times}
\usepackage{latexsym}
\usepackage{amsmath}
\usepackage{amsthm}
\usepackage{amsfonts}
\usepackage{comment}
\newtheorem{definition}{Definition}

\newtheorem{property}{Property}
\newtheorem{method}{Method}
\newtheorem{corollary}{Corollary}

\newtheorem{theorem}{Theorem}
\usepackage{hyperref}
\usepackage{verbatim}
\usepackage[T1]{fontenc}
\usepackage[utf8]{inputenc}
\usepackage{microtype}
\microtypesetup{nopatch=footnote}
\usepackage{setspace}
\usepackage{inconsolata}
\usepackage{graphicx}
\usepackage{subcaption}
\usepackage{float}
\usepackage{multirow}
\usepackage{makecell}
\usepackage{booktabs}
\usepackage{amsthm}
\usepackage{setspace}
\usepackage{placeins}
\usepackage[ruled,vlined]{algorithm2e}
\usepackage{enumitem}
\theoremstyle{remark}
\newtheorem{remark}{Remark}
\newcommand{\rev}[1]{\textcolor{black}{#1}}

\title{IMPACT: Importance-Aware Activation Space Reconstruction}

\author{
Md Mokarram Chowdhury\textsuperscript{1}\thanks{Equal contribution.},
Daniel Agyei Asante\textsuperscript{1},
Ernie Chang\textsuperscript{2},
Yang Li\textsuperscript{1}\footnotemark[1]\thanks{Corresponding author. Address: 2434 Osborn Dr, Ames, IA 50011, United States. Email: jerryyangli@gmail.com.} \\
\textsuperscript{1}Department of Computer Science, Iowa State University, United States \\
\textsuperscript{2}Meta, United States \\
{\tt \{mokarram, dasante, yangli1\}@iastate.edu, erniecyc@meta.com}
}

\begin{document}
\maketitle
\begin{abstract}

Large language models (LLMs) achieve strong performance across diverse domains but remain difficult to deploy in resource-constrained environments due to their size. Low-rank compression is a common remedy, typically minimizing weight reconstruction error under the assumption that weights are low-rank. However, this assumption often does not hold in LLMs. In contrast, LLM activations exhibit a more pronounced low-rank structure, motivating approaches that minimize activation reconstruction error.

This shift alone, however, is not sufficient: different activation dimensions contribute unequally to model performance, and treating them uniformly can lead to accuracy loss. We introduce IMPACT, an importance-aware activation reconstruction framework that links compression to its effect on model performance. IMPACT formulates compression as an optimization problem that integrates activation structure with gradient-based importance, deriving a closed-form solution where reconstruction bases arise from an importance-weighted activation covariance matrix. This yields low-rank compression explicitly optimized for accuracy preservation.
Experiments across multiple models and tasks demonstrate that IMPACT achieves up to 55.4\% greater model size reduction while maintaining accuracy comparable to or better than state-of-the-art baselines.

\end{abstract}
\section{Introduction}

Large language models (LLMs) have achieved remarkable success across a wide range of domains. However, their massive size poses a significant barrier to deployment, particularly in resource-constrained environments. Larger models require more memory, incur slower token throughput, and demand greater computational and energy resources during inference~\cite{Li2023FoldingAttention,li-etal-2025-breaking}. As a result, there is growing urgency to develop compression techniques that can reduce model size while preserving performance.

\begin{figure}[t]
\centering
  \includegraphics[width=1\linewidth]{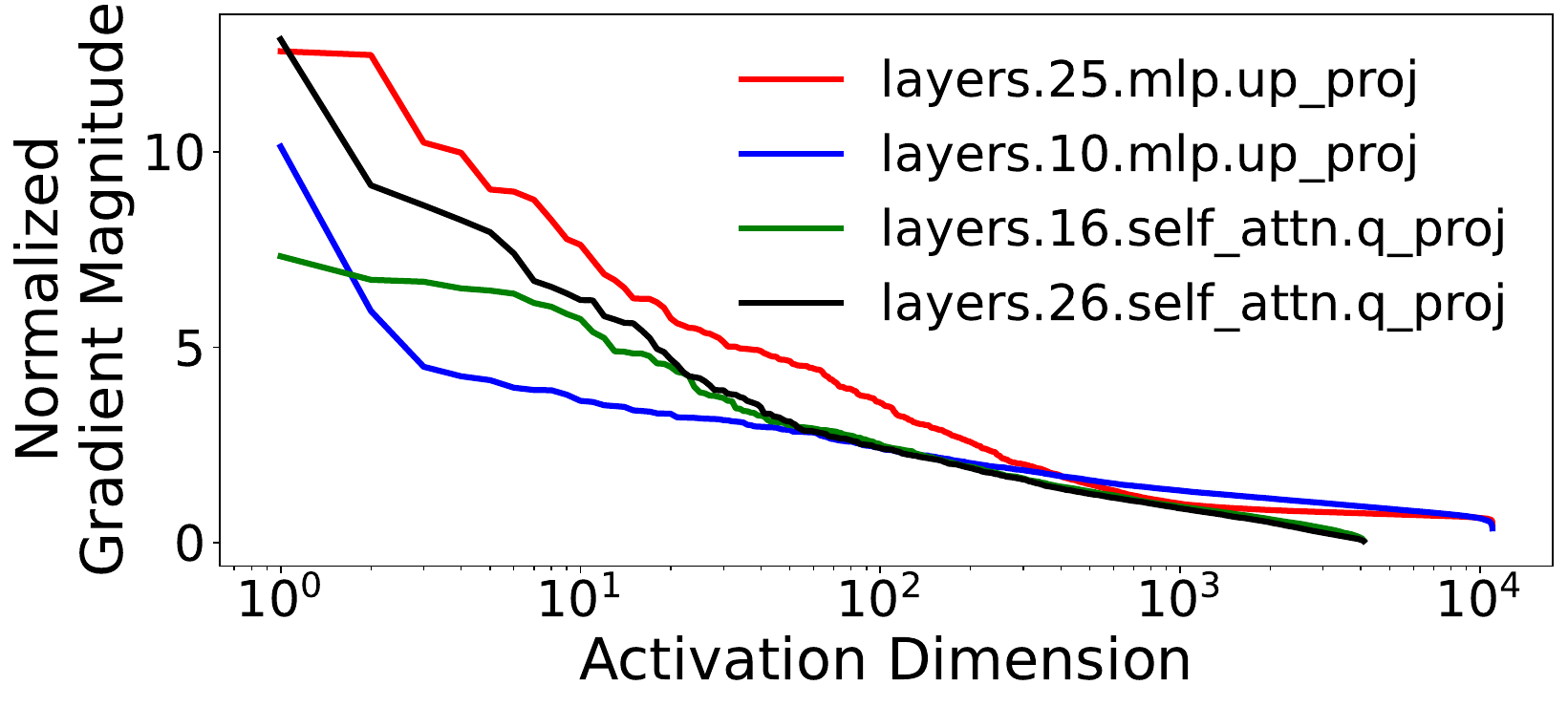}
        \vspace{-20pt}
        \caption{\rev{Normalized average gradient magnitudes across activation dimensions in Llama~2-7B on a mathematical reasoning task. For each linear layer, the output activation is a vector $\mathbf{y} \in \mathbb{R}^d$, where each element $y_i$ represents an activation dimension. Dimensions are sorted in descending order of the normalized average gradient magnitudes. Gradient magnitudes vary substantially across activation dimensions---a pattern consistently observed across other models and tasks.}}
    \label{fig:gradient}
    \vspace{-15pt}
\end{figure}

Low-rank weight matrix compression has emerged as a widely used strategy for model compression \citep{svd1,svd9,svd11_new,svd12_new,svd14_new,svd15_new, basel}. It approximates a weight matrix \( \mathbf{W} \in \mathbb{R}^{m \times n} \) as the product of two smaller matrices, \( \mathbf{W}_1 \in \mathbb{R}^{m \times k} \) and \( \mathbf{W}_2 \in \mathbb{R}^{k \times n} \), where \( k \ll m, n \), thereby reducing the number of parameters. Classical methods select \( \mathbf{W}_1 \) and \( \mathbf{W}_2 \) to minimize the reconstruction error \( \| \mathbf{W} - \mathbf{W}_1 \mathbf{W}_2 \| \), implicitly assuming that the weight matrix itself is low-rank.
However, recent evidence from \citet{afm} reveals that the weight matrices in large-scale models are often not low-rank, limiting the efficacy of direct weight approximation. Interestingly, they observe that the \textit{activations}—the outputs of linear layers—tend to exhibit much stronger low-rank structure. This has led to a shift in focus: instead of minimizing weight reconstruction error, some recent methods~\cite{afm, Drone} minimize the error in reconstructing the activations induced by those weights, achieving better empirical performance.

However, minimizing activation reconstruction error alone is insufficient to guarantee strong model performance. As shown in Figure~\ref{fig:gradient}, different activation dimensions vary widely in their influence on the loss: dimensions associated with large gradients are highly sensitive to reconstruction errors, whereas others contribute little even when poorly reconstructed. Thus, treating all dimensions equally during compression can disproportionately harm those that matter most—leading to greater performance degradation despite low reconstruction error. This raises a critical question:

\begin{center}
\textit{How can we align activation reconstruction decisions with their actual performance impact?}
\end{center}

Answering this question requires a principled framework that explicitly connects weight compression, activation reconstruction, and their contribution to model performance—a connection that has not yet been systematically explored in the context of low-rank weight matrix compression.

To this end, we propose IMPACT, a theoretical framework that guides \underline{imp}ortance-aware \underline{act}ivation space reconstruction. IMPACT rigorously analyzes the relationship among weight compression, activation reconstruction, and model performance. By explicitly linking reconstruction decisions to their performance impact, it provides principled guidance for selecting weights to minimize performance degradation. The framework is grounded in a formal optimization formulation, which is transformed into a more tractable domain and solved through a rigorous analytical derivation using the Lagrange multiplier method. Despite the complexity of the derivation, the final result is remarkably simple: the optimal activation reconstruction bases are the eigenvectors of an importance-weighted activation covariance matrix \( \mathbf{C} = \mathrm{Cov}(\mathbf{y}) \odot \mathbf{M} \), where \( \mathrm{Cov}(\mathbf{y}) \) is the covariance matrix of the activations, and \( \mathbf{M} \) is the gradient-informed importance matrix  (Equations~\eqref{eq:eqn10}--\eqref{eq:eqn12}). These eigenvectors yield the weight matrices \( \mathbf{W}_1 \) and \( \mathbf{W}_2 \) that minimize performance loss. We apply IMPACT to compress a variety of models across diverse datasets and show that it achieves up to 55.4\% greater size reduction while maintaining performance comparable to state-of-the-art baselines.

This paper makes the following contributions:
\begin{itemize}
\vspace{-5pt}
\item We introduce IMPACT, a principled theoretical framework that formally characterizes the relationship between activation reconstruction and its effect on model performance. To our knowledge, this is the first framework within the scope of low-rank weight matrix compression that directly links activation reconstruction choices to model performance.
\vspace{-5pt}
\item We derive a closed-form solution for selecting optimal reconstruction bases using an importance-weighted activation covariance matrix, enabling importance-aware low-rank compression that prioritizes activation dimensions critical to loss minimization.
\vspace{-5pt}
\item 
We empirically validate IMPACT across a broad range of models and tasks, demonstrating that it achieves substantially greater compression while maintaining performance comparable to state-of-the-art methods.
\vspace{-5pt}
\end{itemize}

The final version of this paper is published at ACL 2026 \citep{impact}.

\section{Related Work\label{sec:related}}
Singular value decomposition (SVD) \citep{golub1983matrix} is a widely-used technique for neural network compression, offering low-rank approximations that reduce model size and improve efficiency. Prior work has applied SVD to various network components—convolutional layers, recurrent units, and embeddings—across domains such as language, speech, and vision~\citep{svd1,svd2,svd3,svd4,svd5,svd6,svd7,svd8,svd9,svd10}. More recent efforts extend these techniques to Transformer-based models~\citep{svd11_new, svd12_new, svd13_new, svd14_new, svd15_new, basel}, compressing attention and feedforward layers to enhance memory and compute efficiency. 

Traditional SVD-based methods minimize weight reconstruction error by retaining top singular values and vectors, but this can discard performance-critical information. To address this, FWSVD~\citep{fwsvd} introduces a weighted factorization scheme guided by Fisher information, assigning higher importance to influential weights.

Recent work has proposed weight compression methods that depart from minimizing weight reconstruction error and instead aim to minimize the reconstruction error of layer activations~\citep{afm, asvd, Drone}. Among them, AFM~\citep{afm} explicitly leverages the empirical observation that activations often exhibit stronger low-rank structure than weights, and optimizes the factorized weights to preserve activations. While these approaches have shown improved empirical results, they typically treat all activation dimensions equally, without fully accounting for their varying contribution to model performance.

In contrast, our work focuses on minimizing the impact of activation reconstruction on model performance. Rather than uniformly reducing reconstruction error, we prioritize preserving the most prediction-critical components of the activation. To our knowledge, this is the first framework in the context of low-rank weight matrix compression that explicitly links activation reconstruction choices to their effect on model accuracy.

Beyond efficiency, another important question is how compression affects  trustworthiness. \citet{decomposed_trust} studies the impact of low-rank compression on privacy, adversarial robustness, ethics, and fairness. More broadly, low-rank compression is only one way to improve inference efficiency; other techniques include, for example, quantization, pruning, and token compression. Sheared LLaMA~\citep{xia2024sheared} and KiToke~\citep{KiToke} illustrate the latter two directions.

\section{The IMPACT\label{sec:method} Framework}


In this section, we present IMPACT, our activation reconstruction-based model compression framework. IMPACT identifies a set of directions that enable importance-aware activation reconstruction, minimizing compression-induced performance degradation. We first formulate the optimization problem and define the compression directions in Section~\ref{sec:sec3_1}, laying the foundation for performance-preserving reconstruction. Sections~\ref{sec:sec3_2} to \ref{sec:sec3_6} develop a systematic solution by transforming the activation space, deriving optimal directions via constrained optimization, and constructing the compressed model. Section~\ref{sec:sec3_7} provides a complete algorithm and implementation details of the IMPACT framework.

\subsection{Defining the Objective Function} 
\label{sec:sec3_1}

Let $\mathbf{y} \in \mathbb{R}^d$ be the activation produced by a specific layer of the model for a single input sample. We aim to identify a set of $r$ orthonormal vectors $\{\mathbf{u}_1, \dots, \mathbf{u}_r\}$ that define the directions used to reconstruct activations. Each vector satisfies $\|\mathbf{u}_k\| = 1$ and $\mathbf{u}_i \perp \mathbf{u}_j$ for $i\neq j$, with indices $i,j,k \in \{1, \dots, r\}$. The reconstructed activation is denoted by $\mathbf{\hat{y}}$.

Our objective is to select $\{\mathbf{u}_k\}$ such that the reconstructed activation $\mathbf{\hat{y}}$ closely approximates the original activation $\mathbf{y}$ while preserving model performance. To this end, we define the following objective function:

{\vspace{-10pt}
\small%
\begin{equation}
    \label{eq:eqn1}
    \begin{aligned}
        \min f(\{\mathbf{u}_k\}) = \alpha \mathbb{E}\!\left[\|\mathbf{y} - \mathbf{\hat{y}} \|^2\right] 
    + \beta \mathbb{E}\!\left[(\ell(\mathbf{y}) - \ell(\mathbf{\hat{y}}))^2 \right]
    \end{aligned}
\end{equation}}%

This objective comprises two terms:
\begin{itemize}
    \vspace{-5pt}
    \item {\small $\alpha \mathbb{E}[\|\mathbf{y} - \mathbf{\hat{y}} \|^2]$} encourages $\mathbf{\hat{y}}$ to be numerically close to $\mathbf{y}$.
    \vspace{-5pt}
    \item {\small $\beta \mathbb{E} \left[ (\ell(\mathbf{y}) - \ell(\mathbf{\hat{y}}))^2\right]$} penalizes changes in the loss function $\ell$ due to discrepancies between $\mathbf{y}$ and $\mathbf{\hat{y}}$. The values of $\alpha$ and $\beta$ are set in Section~\ref{sec:sec3_2}.
\end{itemize}

\subsection{Bounding the Objective}
\label{sec:sec3_2}
We now derive an upper bound on the original objective function, providing a more tractable alternative for optimization.

\begin{theorem}[Bounding Theorem]
\label{theorem:theorem1}
Suppose the loss function $\ell$ is $C^1$-smooth, and the activation dimension is $d$. Then the objective function in Equation~\eqref{eq:eqn1} is upper bounded by:

{\vspace{-10pt}
\small\begin{equation}
    \label{eq:eqn2}
    \begin{aligned}
        f(\{\mathbf{u}_k\}) \leq  \mathbb{E}\!\left[ \left\| \sqrt{\beta d \, \mathbb{E}\! \left[ \left( \frac{\partial \ell}{\partial \mathbf{y}} \right)^2 \right]^\top + \alpha }  \odot (\mathbf{y} - \mathbf{\hat{y}}) \right\|^2 \right]
    \end{aligned}
\end{equation}}
\end{theorem}
\rev{Here, $\odot$ denotes the Hadamard (elementwise) product.} The proof is presented in Appendix \ref{appendix:A2}. To simplify the expression and balance the two components of the objective, we set the parameters: 

{
\small{
\[
\alpha = \eta , \quad \beta = \frac{1 - \eta}{\mathbb{E}\!\left[ \left\| \frac{\partial \ell}{\partial \mathbf{y}} \right\|^2\right]}
\]
}}%
\rev{We study the effect of $\eta$ in Appendix~\ref{appendix:abolation}.} Substituting these values into the upper bound yields:

{\vspace{-10pt} 
\small \begin{equation}
    \label{eq:eqn3}
    {\scalebox{0.86}{$
    \begin{aligned}
        h(\{\mathbf{u}_k\}) = \mathbb{E}\! \Bigg[ \Bigg\| &\sqrt{
        \frac{(1 - \eta)}{\frac{1}{d} \mathbb{E}\! \left[ \left\| \frac{\partial \ell}{\partial \mathbf{y}} \right\|^2 \right]} 
        \mathbb{E}\! \left[ \left( \frac{\partial \ell}{\partial \mathbf{y}} \right)^2 \right]^\top + \eta }   
         \odot (\mathbf{y} - \mathbf{\hat{y}}) \Bigg\|^2 \Bigg]
    \end{aligned}
    $}}
\end{equation}}%
This gives the inequality:  

{
\small\[
f(\{\mathbf{u}_k\}) \leq h(\{\mathbf{u}_k\})
\]}%
Directly minimizing $f(\{\mathbf{u}_k\})$ under the orthonormal constraints  

{\vspace{-10pt}
\small \[
\|\mathbf{u}_k\| = 1,\quad \mathbf{u}_i \perp \mathbf{u}_j \; \text{for } i \neq j, \quad i,j,k = 1,\dots,r
\]}%
is analytically challenging. Solving this problem via mathematical programming is computationally prohibitive due to the high dimensionality of the variables $\{\mathbf{u}_k\}$, which scale with the model's parameter size. To address this challenge, we instead minimize the upper bound $h(\{\mathbf{u}_k\})$ of the original objective, subject to the same orthonormal constraints. This relaxation yields an optimization problem that admits an efficient analytical solution.

\subsection{Activation Space Transformation}
\label{sec:sec3_3}
To optimize the objective \( h(\{\mathbf{u}_k\}) \) under the orthonormal constraints, we transform the activations into a new space where the optimization problem can be solved analytically. Specifically, we define a transformation coefficient \( \mathbf{a} \) as:

{
\small \begin{equation}
    \label{eq:eqn4}
    \mathbf{a} = \sqrt{(1 - \eta) \frac{\mathbb{E}\!\left[\left( \frac{\partial \ell}{\partial \mathbf{y}} \right)^2 \right]^\top }{\frac{1}{d} \mathbb{E}\! \left[\left\| \frac{\partial \ell}{\partial \mathbf{y}} \right\|^2\right] } + \eta }
\end{equation}} 

Without loss of generality, we assume that the mean-removed reconstructed activations are given by projecting the original activations (after transformation) onto the subspace spanned by $\{\mathbf{u}_k\}$. Formally, we impose:

{\vspace{-13pt}
\small \begin{equation}
    \mathbf{a} \odot \left( \mathbf{\hat{y}} - \mathbb{E}[\mathbf{y}] \right) = \left( \sum_k \mathbf{u}_k \mathbf{u}_k^\top \right) \left(\mathbf{a} \odot \left( \mathbf{y} - \mathbb{E}[\mathbf{y}] \right)\right)
    \label{eq:eqn5}
\end{equation}
}%
Defining the transformed activation $\mathbf{\tilde{y}}$ as

{
\small\[
\mathbf{\tilde{y}} = \mathbf{a}\odot\left(\mathbf{y} - \mathbb{E}\left[\mathbf{y}\right]\right)
\]}%
we can rewrite the objective \( h(\{\mathbf{u}_k\}) \) in a form that depends only on the transformed activation, as stated in the following theorem.

\begin{theorem}[Activation Space Transformation Theorem]
\label{theorem:theorem2}
Given the transformation \( \mathbf{\tilde{y}} = \mathbf{a} \odot \left(\mathbf{y} - \mathbb{E}[\mathbf{y}]\right) \), the objective \( h(\{\mathbf{u}_k\}) \) becomes:

{
\small \[
h(\{\mathbf{u}_k\}) = \mathbb{E}\! \left[ \mathbf{\tilde{y}}^\top \left( \mathbf{I} - \sum_k \mathbf{u}_k \mathbf{u}_k^\top \right) \mathbf{\tilde{y}} \right]
\]}
\end{theorem}
The proof is provided in Appendix \ref{appendix:A3}.

\subsection{Lagrange Formulation and Derivation}
\label{sec:sec3_4}
To solve the constrained optimization problem, we apply the method of Lagrange multipliers. Our goal is to minimize the objective $h(\{\mathbf{u}_k\})$ subject to the normalization constraint $\|\mathbf{u}_k\| = 1$, along with the orthogonality constraints. To enforce this, we define the Lagrangian function:

{\vspace{-14pt}
\small\begin{equation}
    \label{eq:eqn6}
    L(\{\mathbf{u}_k\}) = h(\{\mathbf{u}_k\}) + \sum_{k=1}^{r} \lambda_k \left( \mathbf{u}_k^\top \mathbf{u}_k - 1 \right)
\end{equation}
}%
where $\lambda_k$ is the Lagrange multiplier associated with the constraint $\mathbf{u}_k^\top \mathbf{u}_k = 1$.

We derive the optimality conditions by taking the derivative of $L(\{\mathbf{u}_k\})$ with respect to each $\mathbf{u}_k$ and setting it to zero. Using standard results from matrix calculus, we obtain:

{
\small \begin{equation}
    \frac{\partial L}{\partial \mathbf{u}_k} = - 2 \mathbf{u}_k^\top \mathbb{E} \left[ \mathbf{\tilde{y}}\mathbf{\tilde{y}}^\top \right] + \lambda_k \mathbf{u}_k^\top
    \label{eq:eqn7}
\end{equation}}%
 
 To simplify notations, we define the \textit{importance-weighted activation covariance matrix}:
 
 {
 \small \begin{equation}
     \mathbf{C} = \mathbb{E}\left[\mathbf{\tilde{y}}\mathbf{\tilde{y}}^\top\right]
     \label{eq:eqn8}
 \end{equation}
 }%
 
 Substituting into Equation~\eqref{eq:eqn7}, the optimality condition becomes:
 
 {
 \small \begin{equation}
     - 2 \mathbf{u}_k^\top \mathbf{C} + \lambda_k \mathbf{u}_k^\top = 0
     \label{eq:eqn9}
 \end{equation}}%

The following theorem characterizes the structure of $\mathbf{C}$.

\begin{theorem}[Importance-Weighted Activation Covariance Matrix]
    \label{theorem:theorem3}
    The matrix $\mathbf{C}$ is equal to the Hadamard product of the activation covariance matrix $\mathrm{Cov}(\mathbf{y)}$ and the gradient-informed importance matrix $\mathbf{M}$, i.e.,

    {
    \small \begin{equation}
        \mathbf{C} =  \mathrm{Cov}\!\left(\mathbf{y}\right) \odot \mathbf{M}
        \label{eq:eqn10}
    \end{equation}}%
where

{\vspace{-10pt}
\small \begin{equation}
    \label{eq:eqn11}
    \mathrm{Cov}\!\left(\mathbf{y}\right) = \mathbb{E} \left[ (\mathbf{y} - \mathbb{E}[\mathbf{y}]) (\mathbf{y} - \mathbb{E}[\mathbf{y}])^\top \right]
\end{equation}}%

{\small \begin{equation}
    \label{eq:eqn12}
    \mathbf{M} = \mathbf{a} \mathbf{a}^\top
\end{equation}}%
\end{theorem}

\noindent A detailed derivation is provided in Appendix \ref{appendix:A4}. 

\begin{remark}
    Since both $\mathrm{Cov}\!\left(\mathbf{y}\right)$ and $\mathbf{M}$ are positive semidefinite, their Hadamard product $\mathbf{C}$ is also positive semidefinite by the Schur product theorem~\cite{zhang2006schur}.  
\end{remark}

\begin{remark}
Because $\mathrm{Cov}\!\left(\mathbf{y}\right)$ and $\mathbf{M}$ are symmetric,  $\mathbf{C}$ is symmetric as well. 
\end{remark}

\rev{To illustrate the importance-aware activation reconstruction enabled by the gradient-informed importance matrix $\mathbf{M}$, we present visual heatmaps and detailed analysis in Appendix~\ref{subsec:appendix_importance_matrix}}.


\subsection{Reconstruction Direction}  
\label{sec:sec3_5}
From Equation~\eqref{eq:eqn9} and the fact that $\mathbf{C}$ is real and symmetric, we have:

{\small \begin{equation}  
    \mathbf{C} \mathbf{u}_k = \lambda_k \mathbf{u}_k  
    \label{eq:eqn13}  
\end{equation}}%
This implies that each reconstruction direction $\mathbf{u}_k$ is an eigenvector of $\mathbf{C}$. We formalize this result in the following theorem.

\begin{theorem}[Reconstruction Direction Theorem]  
    \label{theorem:theorem4}  
    To minimize the objective \(h(\{\mathbf{u}_k\})\) under orthonormality constraints, the optimal \(k^\mathrm{th}\) 
reconstruction direction \(\mathbf{u}_k\) is the eigenvector corresponding to the \(k^\mathrm{th}\) largest eigenvalue of the importance-weighted activation covariance matrix \(\mathbf{C}\).  
\end{theorem}  

A full derivation is provided in Appendix~\ref{appendix:A5}.  

The reconstruction directions $\{\mathbf{u}_k\}_{k=1}^{r}$ are obtained by selecting and normalizing the top $r$ eigenvectors of $\mathbf{C}$. These eigenvectors are guaranteed to be  orthogonal.


\subsection{Compressed Model Representation}
\label{sec:sec3_6}
After obtaining the reconstruction directions $\{\mathbf{u}_k\}_{k=1}^{r}$, we construct the compressed model accordingly. 

Given the relationship between the original activation $\mathbf{y}$ and the reconstructed activation $\mathbf{\hat{y}}$:

{\vspace{-10pt}
\small
\[
\mathbf{a} \odot \left( \mathbf{\hat{y}} - \mathbb{E}[\mathbf{y}] \right) = \left( \sum_k \mathbf{u}_k \mathbf{u}_k^\top \right) \left(\mathbf{a} \odot \left( \mathbf{y} - \mathbb{E}[\mathbf{y}] \right) \right)
\]}%
and the fact that
{\small
$\mathbf{y} = \mathbf{W} \mathbf{x} + \mathbf{b}$,
}%
where {\small $\mathbf{W}$} and {\small $\mathbf{b}$} 
are the original layer's weight matrix and bias, we can express the reconstructed activation $\mathbf{\hat{y}}$ as follows: 

\begin{theorem}[Activation Reconstruction Theorem]
    \label{theorem:theorem5}
    The reconstructed activation $\mathbf{\hat{y}}$, which satisfies the projection condition
    
    {\small \vspace{-8pt}
    \[
        \mathbf{a} \odot \left( \mathbf{\hat{y}} - \mathbb{E}[\mathbf{y}] \right) = \left( \sum_k \mathbf{u}_k \mathbf{u}_k^ \top \right) \left(\mathbf{a} \odot \left( \mathbf{y} - \mathbb{E}[\mathbf{y}] \right)\right)
    \]}%
    is given by:
    
    {
    \small \vspace{-10pt}
    \begin{equation*}
    \begin{aligned}
        \mathbf{\hat{y}} &=  \left[ \mathbf{U} \oslash (\mathbf{a} \cdot \mathbf{1}_r^\top) \right] \left[ (\mathbf{U} \odot (\mathbf{a} \cdot \mathbf{1}_r^\top))^\top \mathbf{W}\right] \mathbf{x} \\
         &+ \mathbb{E}[\mathbf{y}] + (\mathbf{U} \mathbf{U}^\top \odot (\frac{1}{\mathbf{a}} \cdot \mathbf{a}^\top)) (\mathbf{b} - \mathbb{E}[\mathbf{y}])
    \end{aligned}
    \end{equation*}
    }%
    where $\mathbf{U} = [\mathbf{u}_1,\dots,\mathbf{u}_r]$, $\mathbf{1}_r$ is an $r$-dimensional column vector of ones, $\mathbf{W}$ and $\mathbf{b}$ are the original layer's weight matrix and bias, $\mathbf{x}$ is the input activation, and $\oslash$ denotes element-wise division.
\end{theorem}

 A full derivation is provided in Appendix~\ref{appendix:A6}.

Based on this result, the compressed layer is implemented using two linear layers:

\begin{itemize}
    \item The first layer has a weight matrix
    {\small \[
    \mathbf{W}_1 = (\mathbf{U} \odot (\mathbf{a} \cdot \mathbf{1}_r^\top))^\top \mathbf{W}
    \]}%
    and no bias;
    \item The second layer has a weight matrix
    { \small \[
    \mathbf{W}_2 = \mathbf{U} \oslash (\mathbf{a} \cdot \mathbf{1}_r^\top)
    \]}%
    and bias
    { \small \[
    \mathbf{b'} = \mathbb{E}[\mathbf{y}] + (\mathbf{U} \mathbf{U}^\top \odot (\frac{1}{\mathbf{a}} \cdot \mathbf{a}^\top)) (\mathbf{b} - \mathbb{E}[\mathbf{y}])
    \]}  
\end{itemize}
The compressed layer is expressed as:

{
\small \[
\mathbf{\hat{y}} = \mathbf{W}_2 (\mathbf{W}_1 \mathbf{x}) + \mathbf{b'}
\]}

\subsection{IMPACT Algorithm Description}
\label{sec:sec3_7}
Algorithm~\ref{alg:impact} outlines the procedure, which consists of two stages: profiling and compression.

\setlength{\textfloatsep}{10pt plus 1.0pt minus 2.0pt}
\begin{algorithm}[!b]
\begin{small}
\caption{IMPACT Algorithm}
\label{alg:impact}
\KwIn{Model $\mathcal{LM}$}
\KwOut{Compressed Model $\mathcal{LM}'$}
\KwData{Dataset $D$, Keeping Ratio $k$}

\hrulefill\\
\textbf{Stage 1: Profiling}\;
Let $n$ be the total number of samples in $D$\;
\For{each layer $l$ in $\mathcal{LM}$}{
Initialize 
$\mathbb{E}\!\left[\mathbf{y}\mathbf{y}^\top\right]_l = 0, \, \mathbb{E}\!\left[\mathbf{y}\right]_l = 0, \, \mathbb{E}\!\Bigl[\bigl(\tfrac{\partial \ell}{\partial \mathbf{y}}\bigr)^{\!2}\Bigr]_l = 0$\;
}
\For{each sample $s \in D$}{
    Get activation $\mathbf{y}_l$ and gradient $\frac{\partial \ell}{\partial \mathbf{y}_l}$ for layer $l$\;
    \For{each layer $l$ in $\mathcal{LM}$}{
        $\mathbb{E}\!\left[\mathbf{y}\mathbf{y}^\top\right]_l \leftarrow \mathbb{E}\!\left[\mathbf{y}\mathbf{y}^\top\right]_l + \mathbf{y}_l\mathbf{y}^\top_l$\;\vspace{3pt}
        $\mathbb{E}\!\left[\mathbf{y}\right]_l \leftarrow \mathbb{E}\!\left[\mathbf{y}\right]_l + \mathbf{y}_l$\; \vspace{3pt}
        $\mathbb{E}\!\Bigl[\bigl(\tfrac{\partial \ell}{\partial \mathbf{y}}\bigr)^{\!2}\Bigr]_l
    \;\gets\;
    \mathbb{E}\!\Bigl[\bigl(\tfrac{\partial \ell}{\partial \mathbf{y}}\bigr)^{\!2}\Bigr]_l
    + \bigl(\tfrac{\partial \ell}{\partial \mathbf{y}_l}\bigr)^{\!2}$\;
        
    }
}
\For{each layer $l$ in $\mathcal{LM}$ }{
    $\mathbb{E}\!\left[\mathbf{y}\mathbf{y}^\top\right]_l \leftarrow \mathbb{E}\!\left[\mathbf{y}\mathbf{y}^\top\right]_l / n$\; \vspace{3pt}
    $\mathbb{E}\!\left[\mathbf{y}\right]_l \leftarrow \mathbb{E}\!\left[\mathbf{y}\right]_l /n$\; \vspace{3pt}
    $\mathrm{Cov}(\mathbf{y})_l \leftarrow \mathbb{E}\!\left[\mathbf{y}\mathbf{y}^\top\right]_l - \mathbb{E}\!\left[\mathbf{y}\right]_l\mathbb{E}\!\left[\mathbf{y}\right]^\top_l$\;
}
\hrulefill\\
\textbf{Stage 2: Compression}\;
\For{each layer $l$ in $\mathcal{LM}$}{
    // For brevity, the subscript $l$ is omitted from the notations presented below.
    
    Compute the transformation coefficient $\mathbf{a}$ based on Equation~\eqref{eq:eqn4}\;
    Compute the gradient-informed importance matrix $\mathbf{M}$ based on Equation~\eqref{eq:eqn12}\;
    Compute the importance-weighted activation covariance matrix $\mathbf{C}$ based on Equation~\eqref{eq:eqn10}\;
    $[\mathbf{U},\mathbf{\Lambda}] = \textrm{eigenvalue\_decomposition}(\mathbf{C})$\;
    // The columns of $\mathbf{U}$ are the eigenvectors of $\mathbf{C}$\; 
    // The vector $\mathbf{\Lambda}$ consists of the eigenvalues of $\mathbf{C}$\;
    Sort the elements of $\mathbf{\Lambda}$ in descending order and reorder $\mathbf{U}$ accordingly\;
    Find smallest $r$ such that
     $\left(\sum_{j=1}^r \sqrt{\Lambda_j}\right) \big / \left(\sum_{j=1}^d \sqrt{\Lambda_j}\right) \geq k/100$\;
    $\mathbf{U} \leftarrow$ First $r$ columns of $\mathbf{U}$ \;
    Substitute the original linear layer with two new linear layers with smaller sizes:
    
    The first new layer has a weight matrix of $(\mathbf{U} \odot (\mathbf{a} \cdot \mathbf{1}_r^\top))^\top \mathbf{W}$ and no bias\;
    The second new layer has a weight matrix of $\mathbf{U} \oslash (\mathbf{a} \cdot {\mathbf{1}_r^\top}$)\ and a bias of $\mathbb{E}\!\left[\mathbf{y}\right] + (\mathbf{U} \mathbf{U}^\top \odot (\frac{1}{\mathbf{a}} \cdot \mathbf{a}^\top)) (\mathbf{b} - \mathbb{E}\left[\mathbf{y}\right])$;}

\Return Compressed Model $\mathcal{LM}'$\;
\end{small}
\end{algorithm}

\begin{figure*}[t]
    \centering
    \begin{subfigure}{.44\textwidth}
        \includegraphics[width=\linewidth]{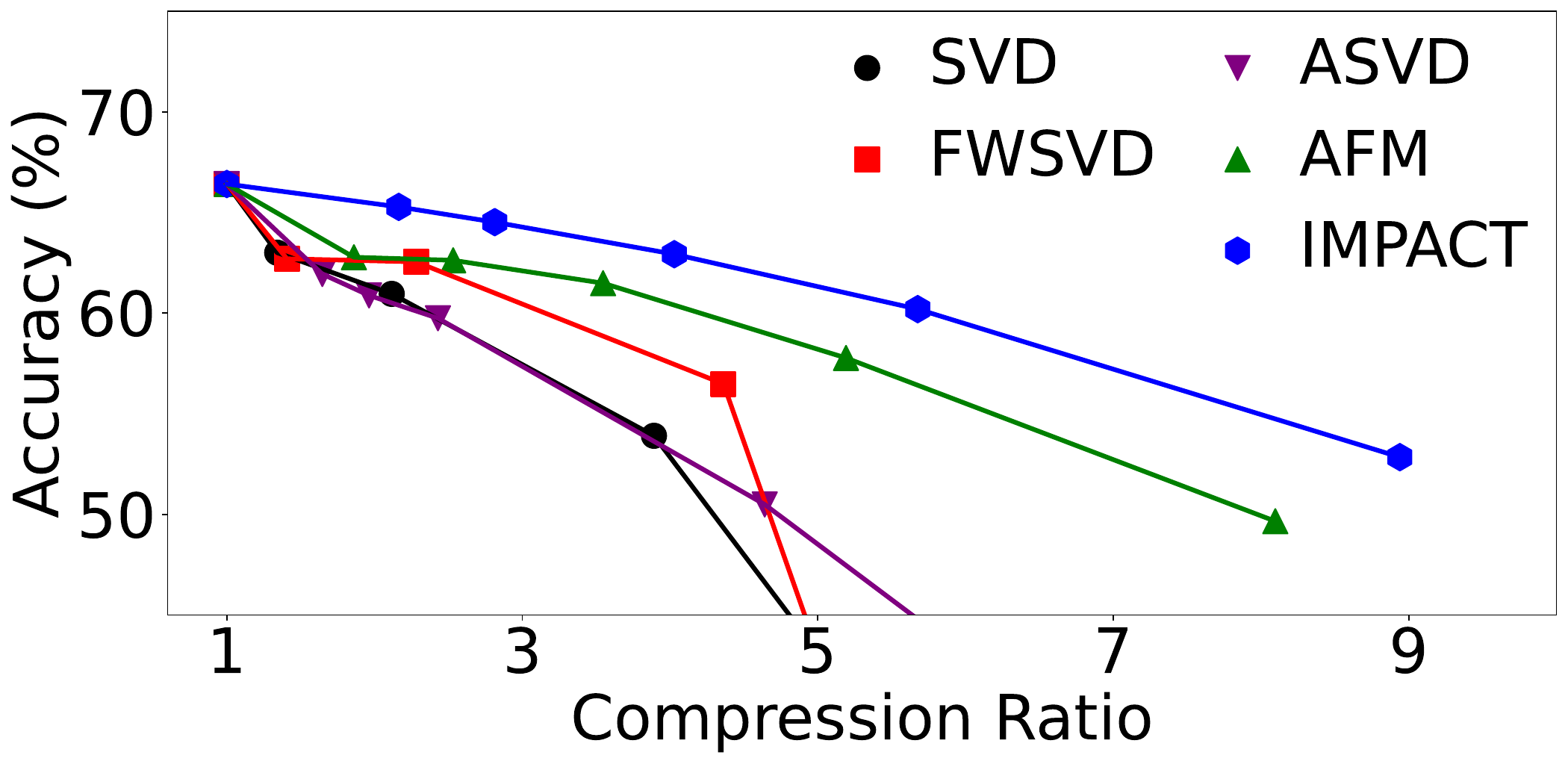}
        \caption{GSM8K}
    \end{subfigure}\hfill
    \begin{subfigure}{.44\textwidth}
        \includegraphics[width=\linewidth]{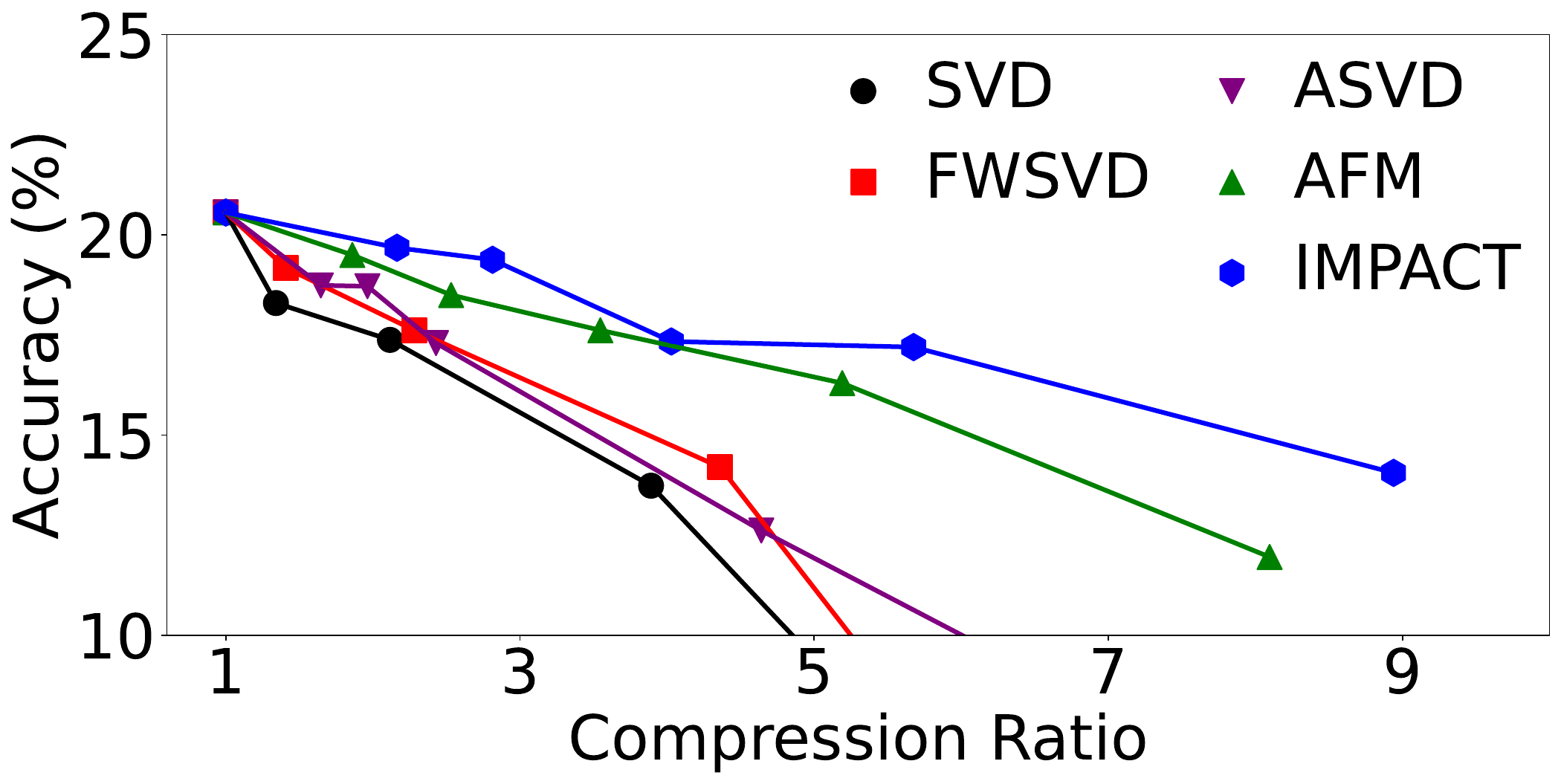}
        \caption{MATH}
    \end{subfigure}
    \caption{Pass@1 accuracy and model size of Llama 2-7B compressed with various low-rank algorithms on the mathematical reasoning task. Exact values are listed in Table~\ref{tbl:7b-math}.}
    \label{fig:7b-math}
\end{figure*}
\begin{figure*}[t]
    \centering
    \begin{subfigure}{.44\textwidth}
        \includegraphics[width=\linewidth]{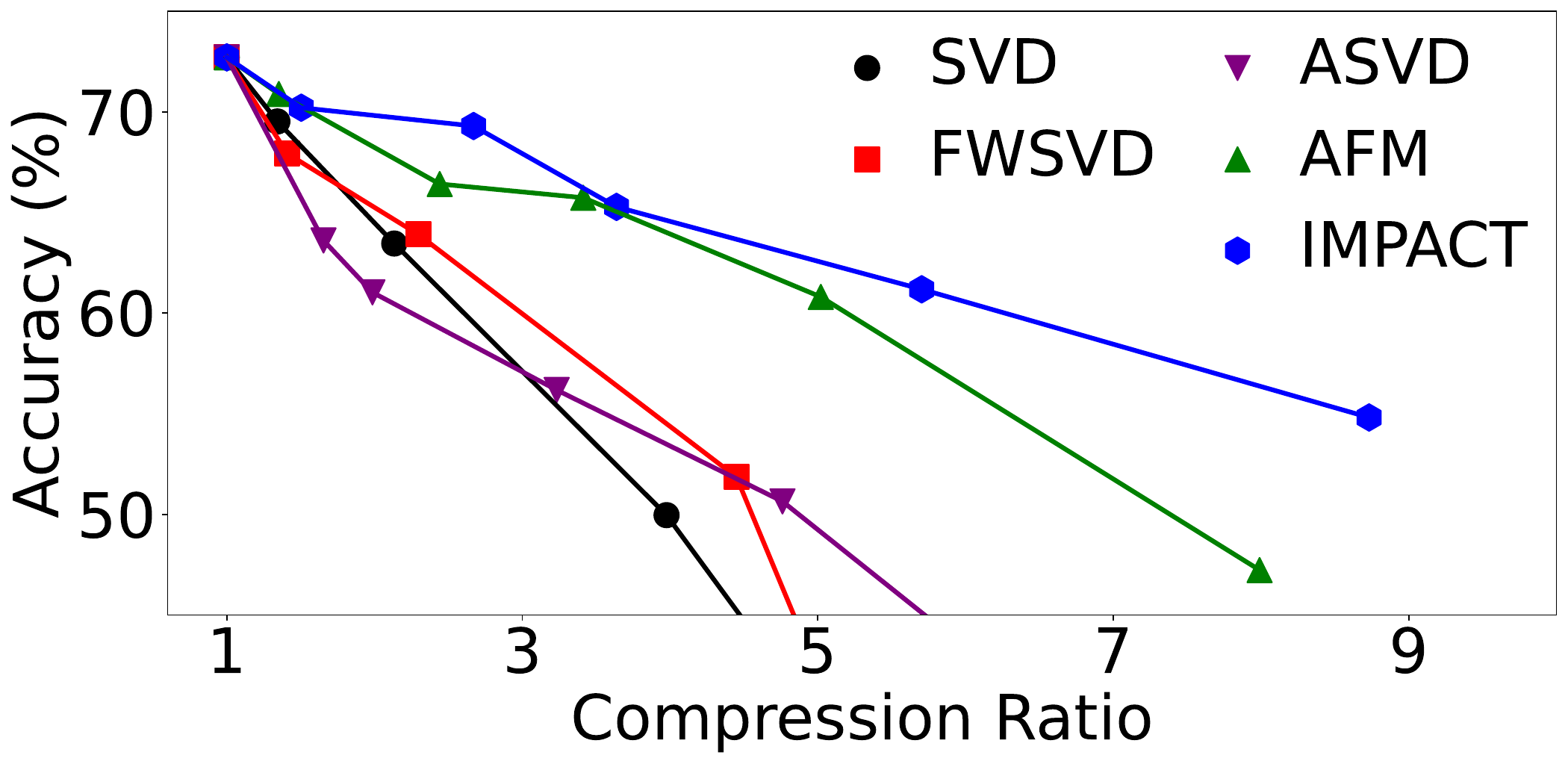}
        \caption{GSM8K}
    \end{subfigure}\hfill
    \begin{subfigure}{.44\textwidth}
        \includegraphics[width=\linewidth]{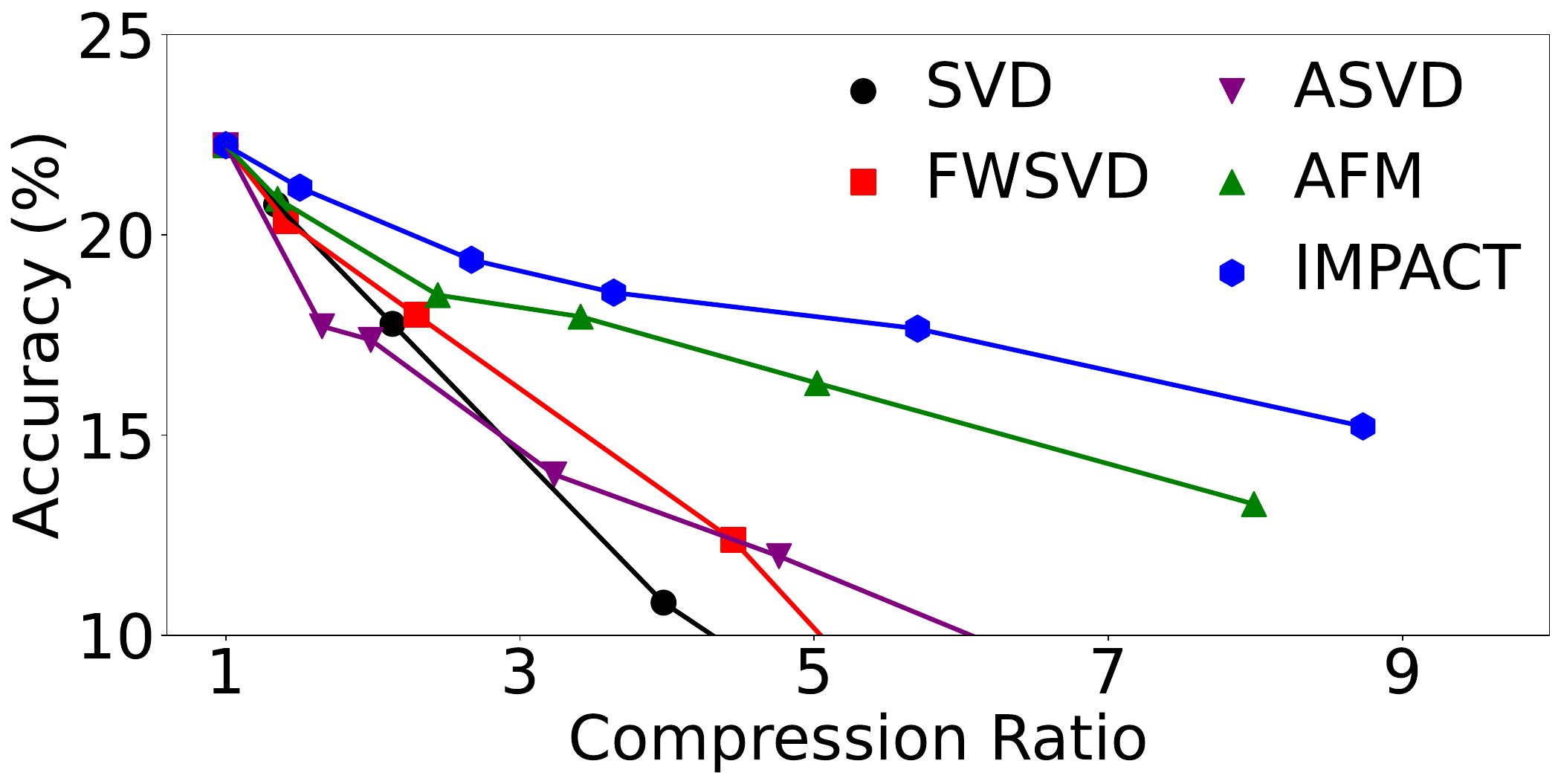}
        \caption{MATH}
    \end{subfigure}
    \caption{Pass@1 accuracy and model size of Llama 2-13B compressed with various low-rank algorithms on the mathematical reasoning task. Exact values are listed in Table~\ref{tbl:13b-math}.}
    \label{fig:13b-math}
\end{figure*}
\subsubsection{Profiling Stage}
The algorithm gathers activation and gradient statistics for each linear layer in the model. Specifically, it computes the mean activation, the activation covariance matrix, and the mean squared gradient with respect to the activations. These statistics form the basis for the subsequent compression step.

\subsubsection{Compression Stage} 

Using the collected statistics, the algorithm constructs the importance-weighted activation covariance matrix $\mathbf{C}$ for each linear layer by applying a Hadamard product between the activation covariance and the gradient-informed importance matrix. Eigenvalue decomposition is then performed on $\mathbf{C}$ to extract the top eigenvectors, which define the compression directions. Each original linear layer is subsequently replaced by a pair of smaller linear layers designed to preserve model performance.

\begin{figure*}[t]
    \centering
    \begin{subfigure}{.44\textwidth}
        \includegraphics[width=\linewidth]{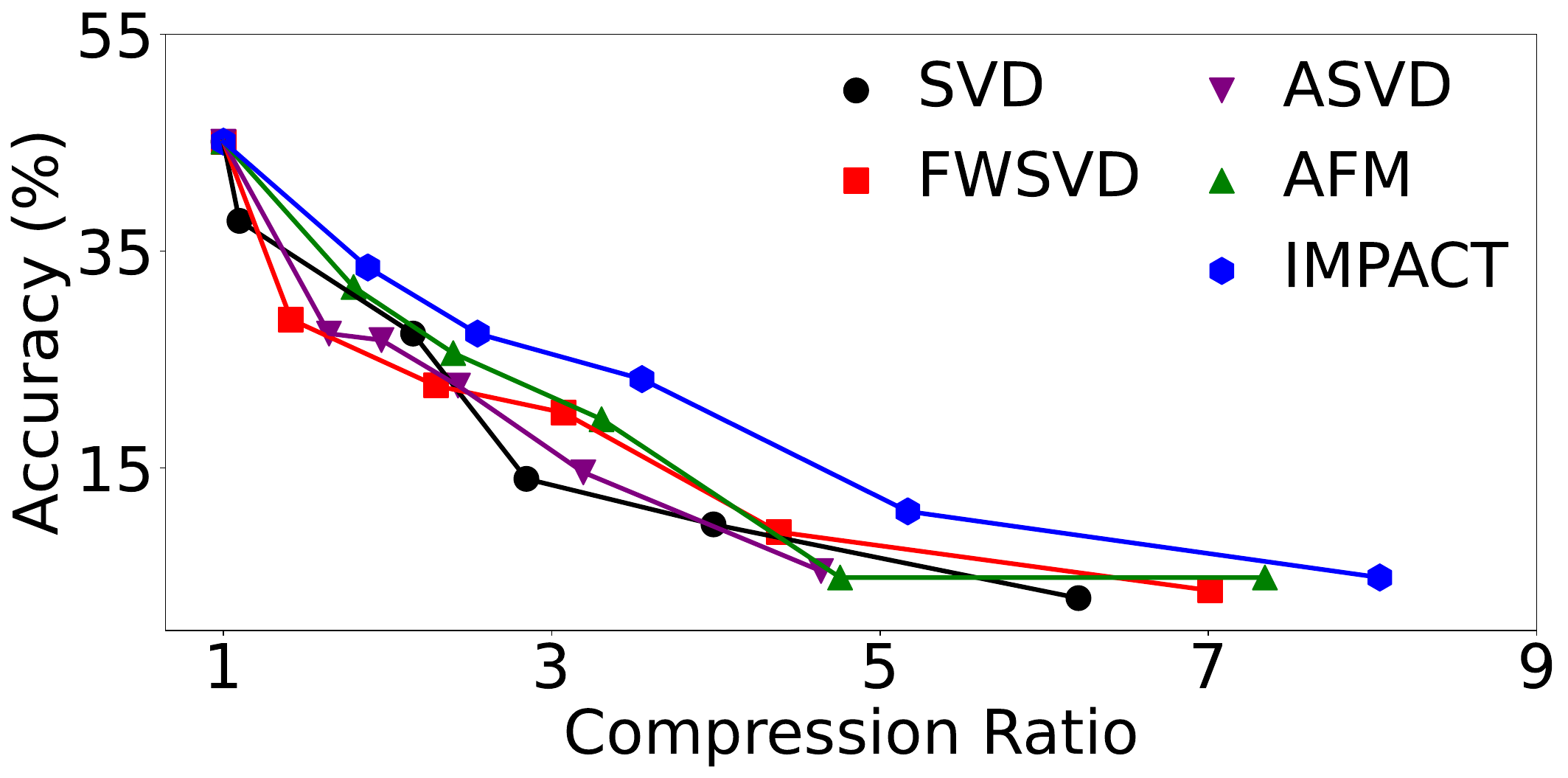}
        \caption{HumanEval}
    \end{subfigure}\hfill
    \begin{subfigure}{.44\textwidth}
        \includegraphics[width=\linewidth]{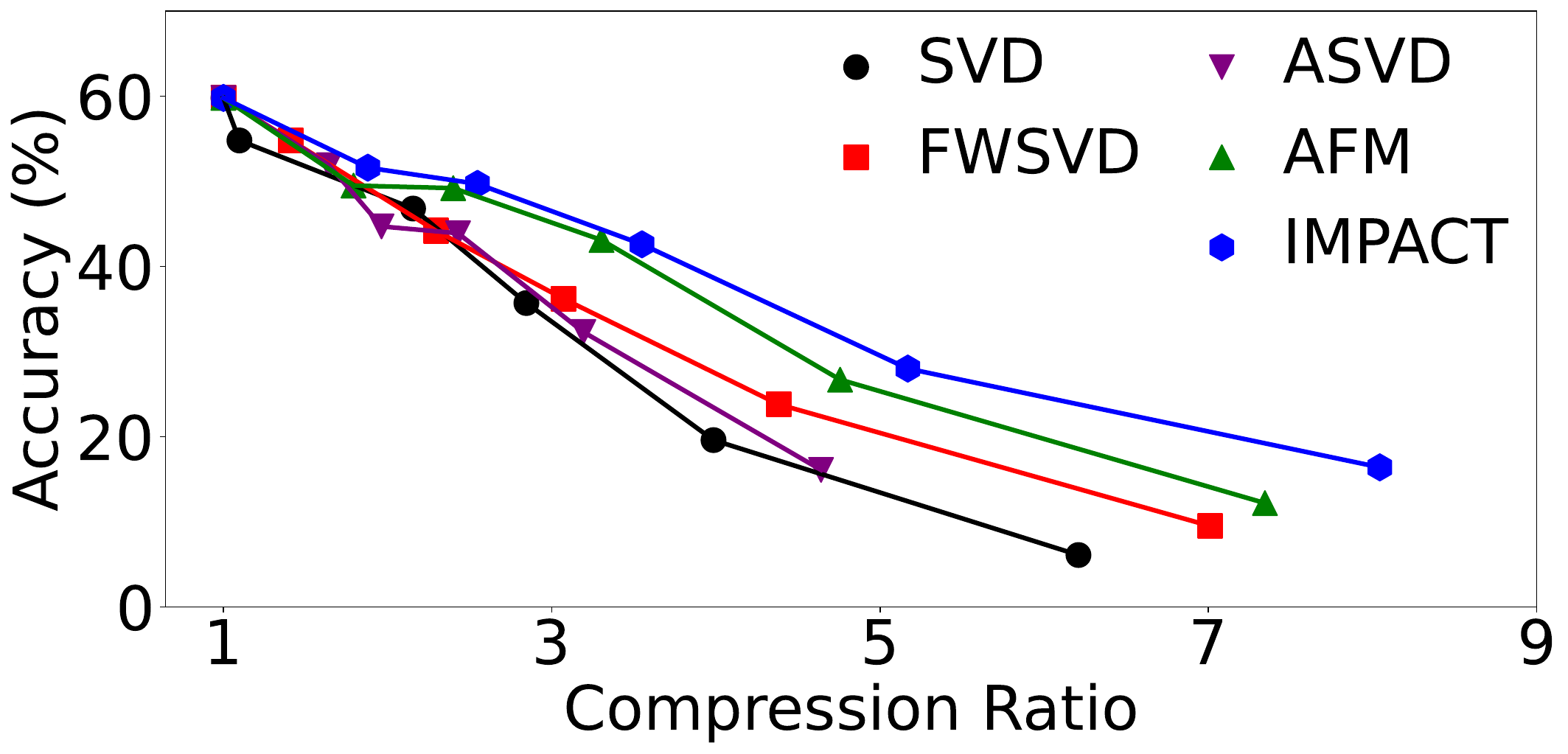}
        \caption{MBPP}
    \end{subfigure}
    \vspace{-2pt}
    \caption{
    Pass@1 accuracy and model size of CodeLlama-7B compressed with various low-rank algorithms on the code generation task. Exact values are listed in Table~\ref{tbl:7b-prog}.}
    \label{fig:7b-programming}
\end{figure*}
\begin{figure*}[t]
    \centering
    \begin{subfigure}{.44\textwidth}
        \includegraphics[width=\linewidth]{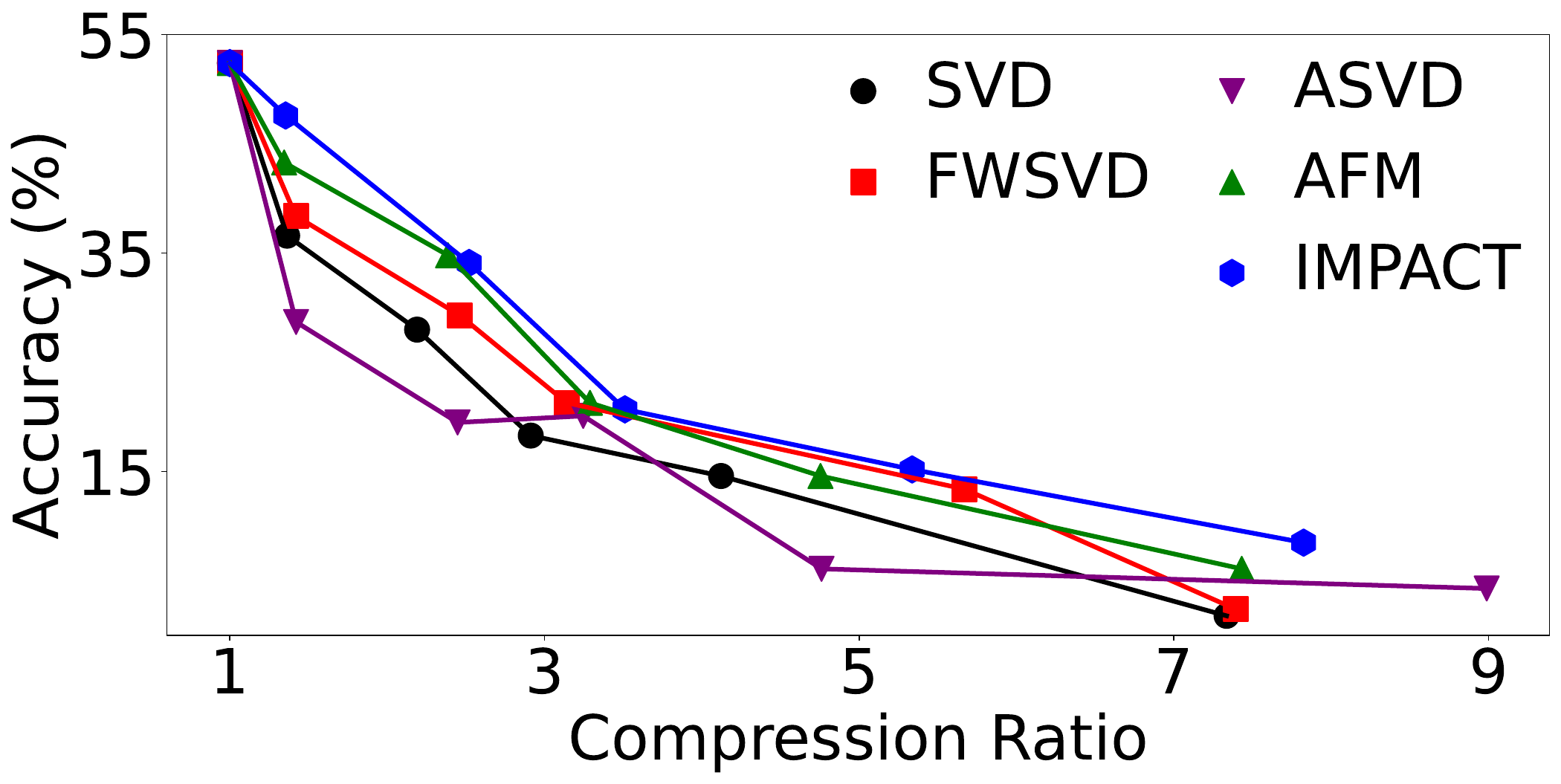}
        \caption{HumanEval}
    \end{subfigure}\hfill
    \begin{subfigure}{.44\textwidth}
        \includegraphics[width=\linewidth]{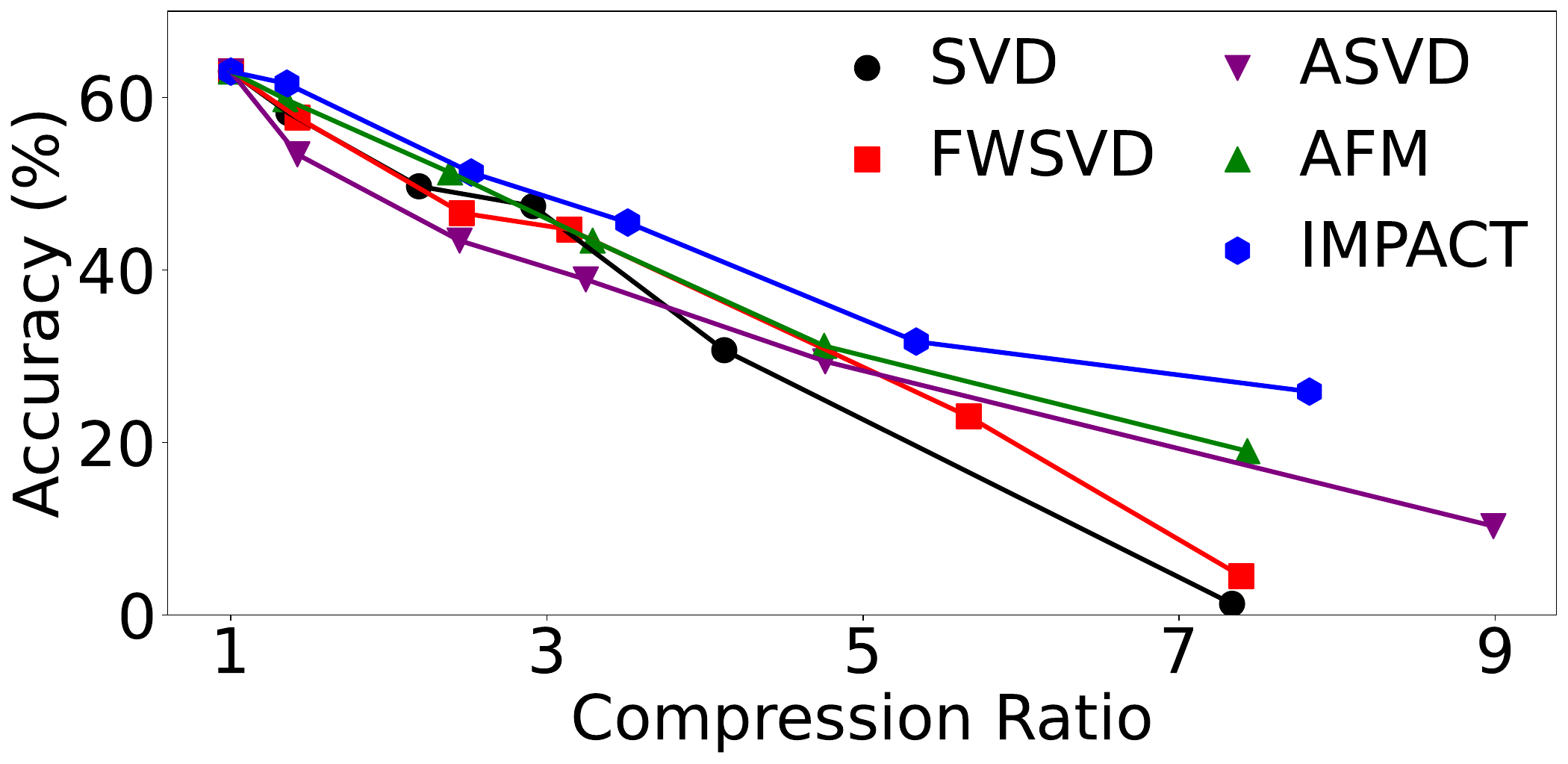}
        \caption{MBPP}
    \end{subfigure}
    \vspace{-2pt}
    \caption{
    Pass@1 accuracy and model size of CodeLlama-13B compressed with various low-rank algorithms on the code generation task. Exact values are listed in Table~\ref{tbl:13b-prog}.}
    \label{fig:13b-programming}
\end{figure*}
\section{Experiments}

\subsection{Evaluation Methodology\label{subsec:methodology}}
We evaluate the effectiveness of low-rank compression algorithms on two tasks: mathematical reasoning and code generation. For mathematical reasoning, we use the Llama 2-7B and -13B models~\citep{llama2}; for code generation, we use CodeLlama-7B and -13B~\citep{codellama}. Each model is first finetuned on a task-specific finetuning set, then compressed using a low-rank method, and finally undergoes post-compression finetuning before evaluation. For fair comparison, we applied the same supervised fine-tuning (SFT) protocol to all methods, including IMPACT and all baselines. We used AdamW with a learning rate of 2e-5 for Llama 2-7B and CodeLlama-7B, and 1e-5 for Llama 2-13B and CodeLlama-13B, along with cosine learning rate scheduling (warmup ratio 3\%), BF16 mixed precision, and FSDP. The total batch size was 128 for both Llama 2-7B and CodeLlama-7B (2 GPUs × 64 samples per GPU), and for Llama 2-13B and CodeLlama-13B (4 GPUs × 32 samples per GPU). \rev{All profiling was performed as a one-time post-training pass on a single NVIDIA A100 GPU. Additional details on the profiling cost and stability with respect to calibration data are provided in
Appendix~\ref{app:profiling-cost}.}


We employ distinct datasets for calibration and evaluation. For calibration, we use MetaMathQA-395K~\cite{metamath} for mathematical reasoning and Code-Instructions-120K~\cite{codeinstructions120k} for code generation. For evaluation, on the mathematical reasoning task, we evaluate on GSM8K~\citep{gsm8k} and Hendrycks’ MATH~\citep{hendrycks}; on code generation, we use MBPP~\citep{mbpp} and HumanEval~\citep{humaneval} as evaluation sets.

We compare our proposed method, IMPACT, against state-of-the-art low-rank compression techniques, including SVD~\citep{svd1,svd10,svd11_new,svd12_new,svd13_new,svd14_new,svd15_new,svd16_new}, a widely used matrix factorization method; FWSVD~\citep{fwsvd}, which incorporates weight importance; and activation-aware approaches such as ASVD~\citep{asvd} and AFM~\citep{afm}.


Beyond low-rank compression methods, we also benchmark IMPACT against compression techniques from other paradigms, including QLoRA~\citep{QLoRA}, a quantization approach that finetunes low-rank adapters, and FLAP~\citep{FLAP}, a pruning method that removes weights based on magnitude and activation variance. These comparisons highlight IMPACT’s robustness and effectiveness across diverse compression strategies.

\subsection{Evaluation of Low-Rank Compression for Mathematical Reasoning}
 Figures~\ref{fig:7b-math} and~\ref{fig:13b-math} show the performance of low-rank compression methods on Llama 2-7B and -13B for the mathematical reasoning task. We evaluate the Pass@1 accuracy of the models across a range of compression ratios (the ratio of the original model size to the compressed model size). The precise numerical values corresponding to these performance curves are reported in Tables~\ref{tbl:7b-math} and~\ref{tbl:13b-math}. 
 Our proposed method, IMPACT, consistently outperforms SVD, FWSVD, ASVD, and AFM across all compression ratios, achieving greater compression while maintaining comparable or superior accuracy.



 On Llama~2-7B, IMPACT achieves up to 55.4\% greater size reduction than the strongest baseline (AFM) on GSM8K, and up to 31.7\% more on MATH, while maintaining the same accuracy. Across all evaluated compression ratios, it compresses the model over 40\% more than SVD, FWSVD, and ASVD on both datasets while delivering similar or better performance. Similar patterns are observed for Llama~2-13B, where IMPACT achieves up to 39.8\% more compression than AFM on GSM8K and 36.5\% more on MATH. At compression ratios above 2.5$\times$, IMPACT continues to deliver over 35\% more compression than SVD, FWSVD, and ASVD while maintaining better performance.

\subsection{Evaluation of Low-Rank Compression for Code Generation}

\begin{figure*}[t]
    \centering
    \begin{subfigure}{.44\textwidth}
        \includegraphics[width=\linewidth]{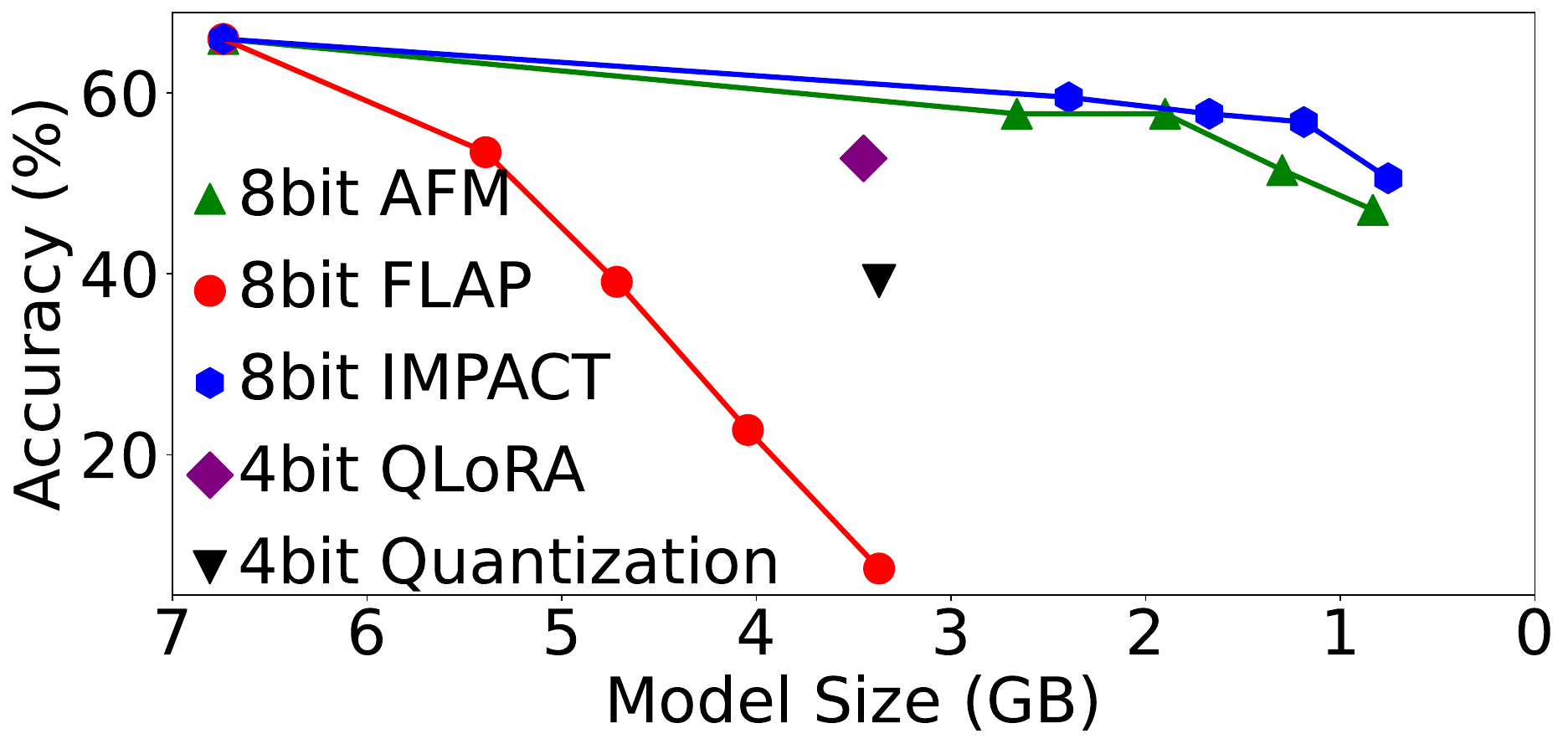}
        \caption{GSM8K}
    \end{subfigure}\hfill
    \begin{subfigure}{.44\textwidth}
        \includegraphics[width=\linewidth]{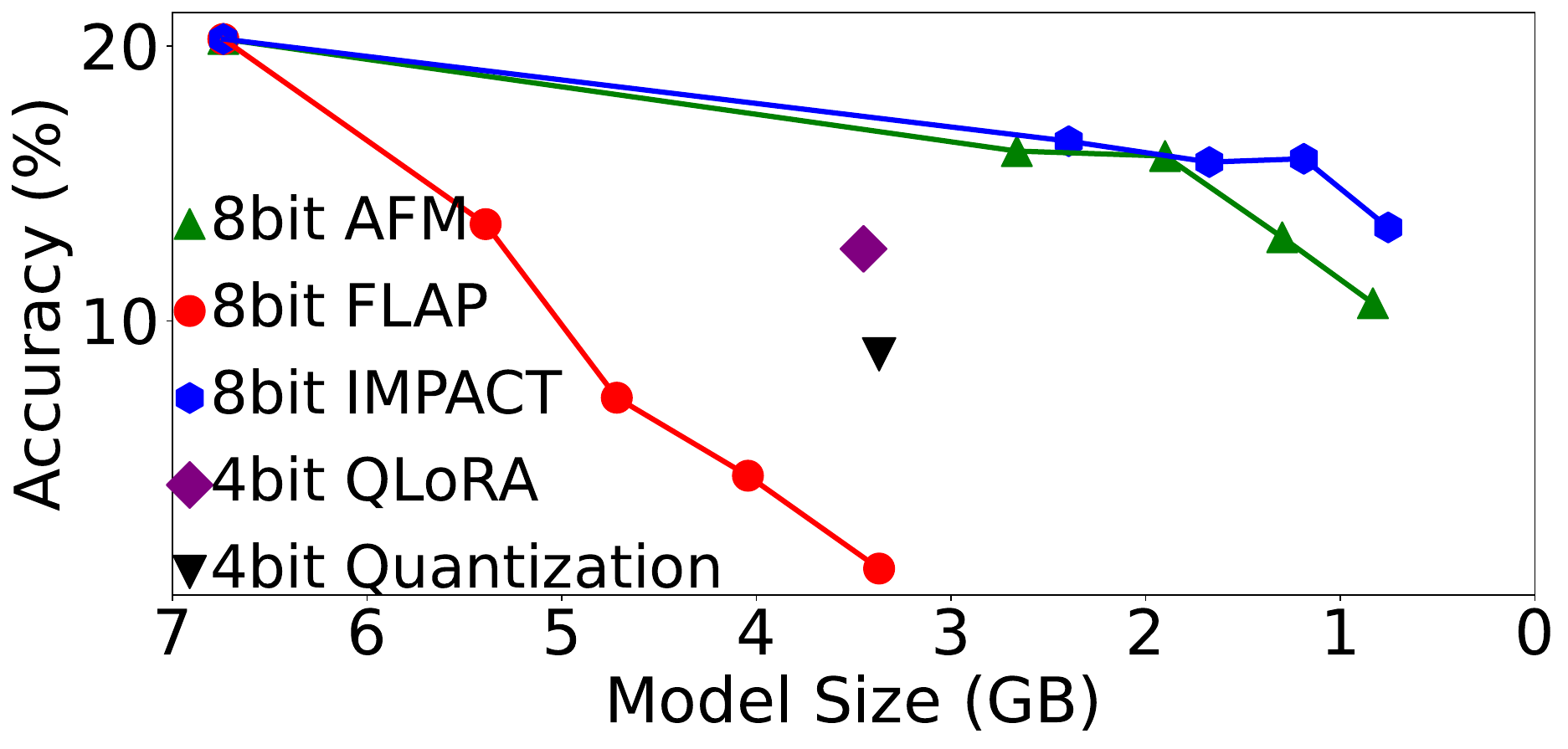}
        \caption{MATH}
    \end{subfigure}
    \caption{Pass@1 accuracy and model size of Llama 2-7B models compressed using quantization alone, as well as in combination with low-rank compression or pruning techniques, evaluated on the mathematical reasoning task.}
    \label{fig:7b-math-quantization}
\end{figure*}

\begin{figure*}[!t]
    \centering
    \begin{subfigure}{.44\textwidth}
        \includegraphics[width=\linewidth]{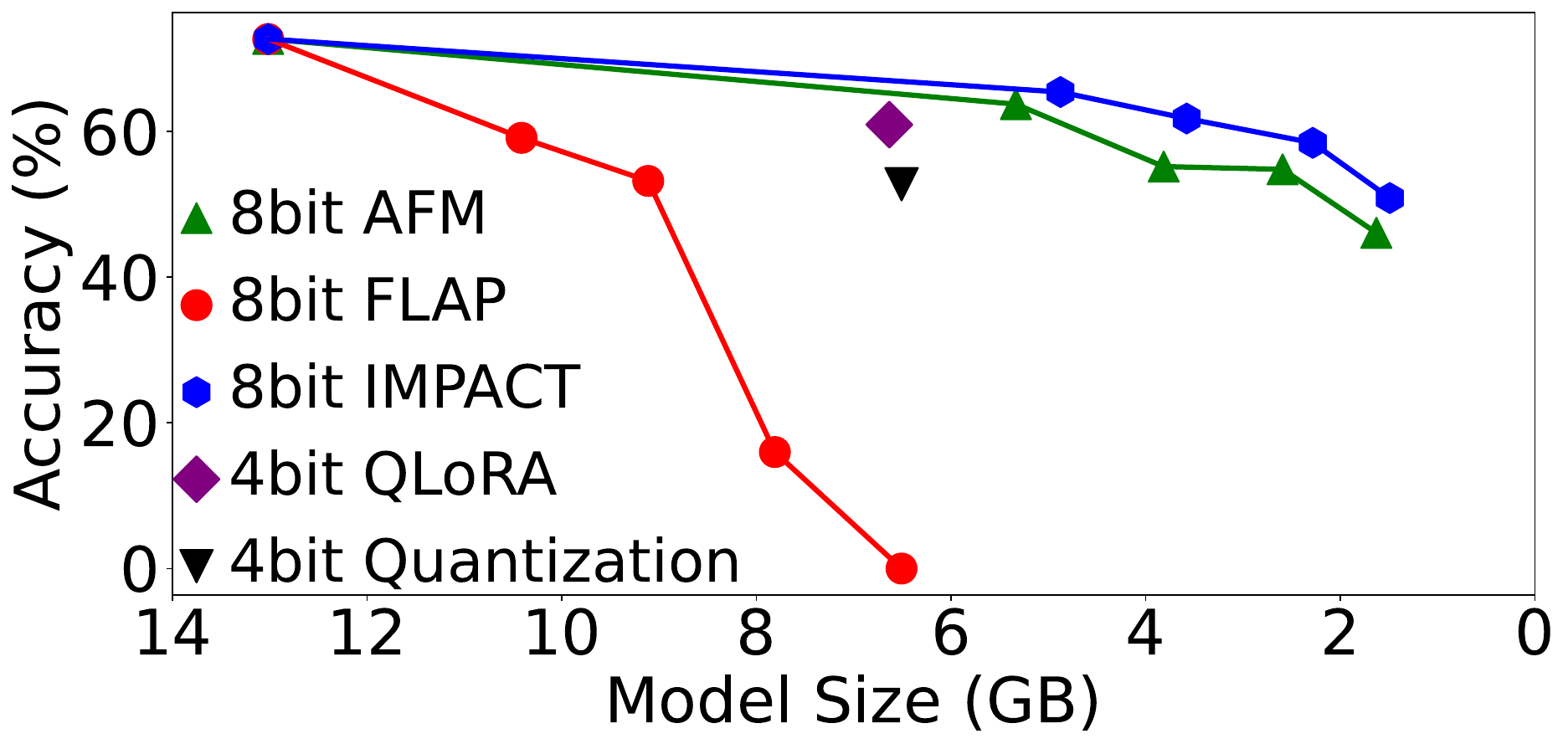}
        \caption{GSM8K}
    \end{subfigure}\hfill
    \begin{subfigure}{.44\textwidth}
        \includegraphics[width=\linewidth]{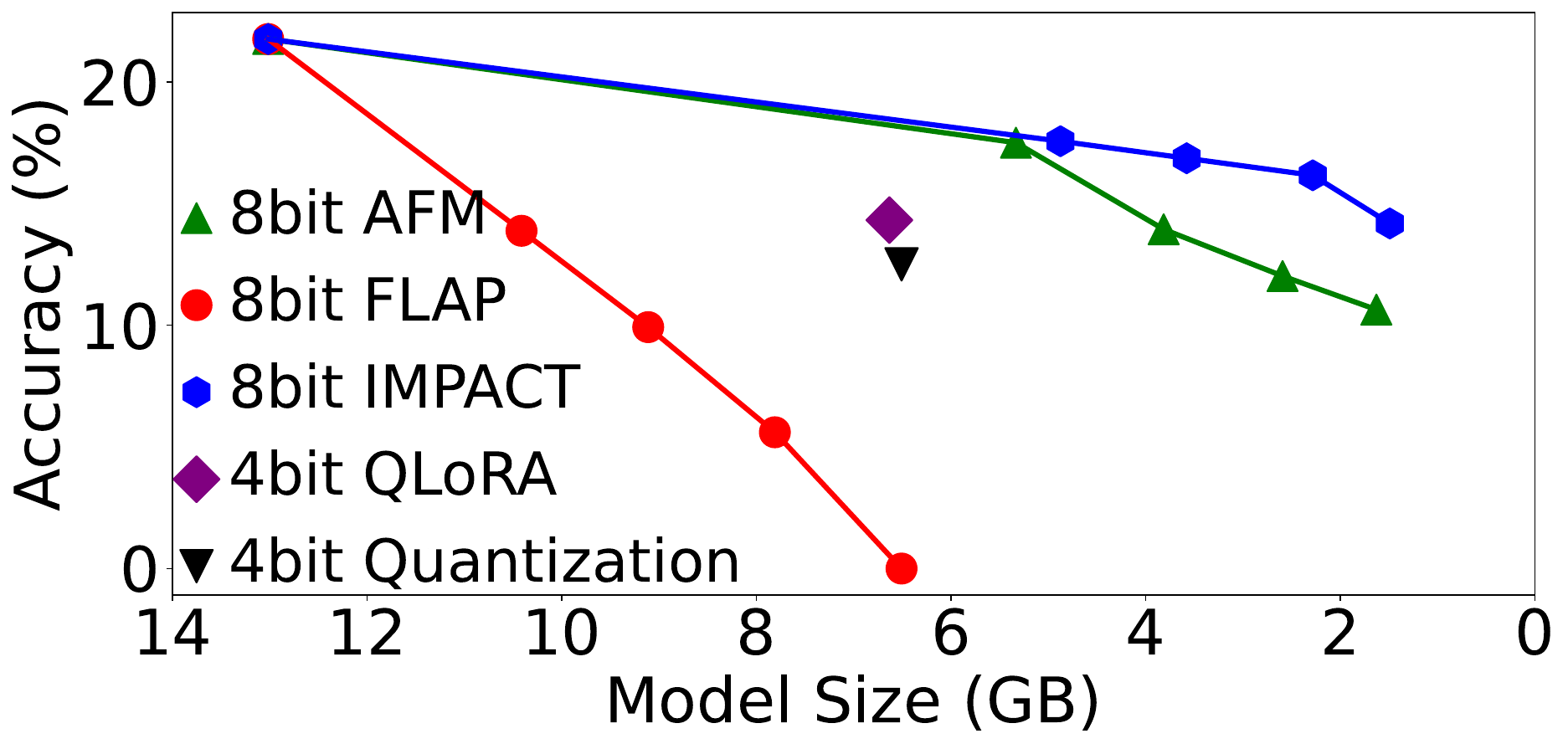}
        \caption{MATH}
    \end{subfigure}
    \caption{Pass@1 accuracy and model size of Llama 2-13B models compressed using quantization alone, as well as in combination with low-rank compression or pruning techniques, evaluated on the mathematical reasoning task. }
    \label{fig:13b-math-quantization}
\end{figure*}

Figures~\ref{fig:7b-programming} and~\ref{fig:13b-programming} show the performance of IMPACT and baselines on CodeLlama-7B and -13B for code generation, with precise numerical values reported in Tables~\ref{tbl:7b-prog} and~\ref{tbl:13b-prog}. We evaluate the Pass@1 accuracy on the HumanEval and MBPP benchmarks across a range of compression ratios.


IMPACT consistently outperforms baseline methods by achieving greater compression while maintaining comparable or superior accuracy on both code generation tasks. On CodeLlama-7B, IMPACT reduces model size by up to 24.2\% more than the best-performing baseline on HumanEval, and by 30.7\% more on MBPP. Similar trends are observed on CodeLlama-13B, where IMPACT achieves up to 18.1\% more compression on HumanEval and 28.0\% more on MBPP compared to the strongest baseline.

\subsection{Integrating IMPACT with Quantization\label{subsec:quantization}
}

Quantization and low-rank compression are distinct model compression techniques grounded in different principles: quantization reduces the precision of model weights, whereas low-rank compression approximates weight matrices as the product of smaller matrices. Quantization generally preserves performance at 8-bit precision or higher but often degrades accuracy at lower precisions like 4-bit. To assess the combined effect of quantization and other compression methods such as low-rank compression and pruning, we integrate 8-bit quantization with IMPACT, \rev{AFM (the strongest low-rank baseline)}, and FLAP to produce compressed models of various sizes.

Results on mathematical reasoning with Llama~2-7B (Figure~\ref{fig:7b-math-quantization}, Tables~\ref{tbl:7b-flap-impact} and~\ref{tbl:qlora-vs-full-model-quantization-13b-and-7b}) show that IMPACT with 8-bit quantization consistently outperforms pure 4-bit quantization, 4-bit QLoRA, \rev{8-bit AFM}, and 8-bit FLAP. For example, 8-bit IMPACT achieves higher accuracy at a model size of 1.19 GB than 8-bit FLAP, 8-bit AFM, and 4-bit QLoRA do at 5.39 GB, 1.30 GB, and 3.45 GB, respectively. Similar trends are observed for Llama~2-13B (Figure~\ref{fig:13b-math-quantization}, Tables~\ref{tbl:13b-flap-impact} and~\ref{tbl:qlora-vs-full-model-quantization-13b-and-7b}), where 8-bit IMPACT outperforms all baselines. These results highlight the superior performance of IMPACT and the benefit of combining low-rank compression with quantization, which yields higher accuracy than either technique alone at the same model size.

\begin{table}[!b]
\centering
\resizebox{0.5\textwidth}{!}{
\begin{tabular}{llrrrrr}
\toprule

\multirow{3}{*}{8-bit FLAP} 
    & \multicolumn{1}{l}{Model Size (GB)} & 6.74 & 5.39 & 4.72 & 4.04 & 3.37 \\ 
    & \multicolumn{1}{l}{GSM8K Acc (\%)} & 66.0 & 53.4 & 39.1 & 22.7 & 7.4  \\
    & \multicolumn{1}{l}{MATH Acc (\%)} & 20.3 & 13.5 & 7.2 & 4.4 & 1.0 \\ \midrule
\multirow{3}{*}{8-bit AFM} 
    & \multicolumn{1}{l}{Model Size (GB)} & 6.74 & 2.66 & 1.90 & 1.30 & 0.83 \\ 
    & \multicolumn{1}{l}{GSM8K Acc (\%)} & 66.0 & 57.7 &  57.7 & 51.5 &  47.1  \\
    & \multicolumn{1}{l}{MATH Acc (\%)} & 20.3 &  16.2 & 16.0 & 13.0 & 10.6 \\ \midrule
\multirow{3}{*}{8-bit IMPACT} 
    & \multicolumn{1}{l}{Model Size (GB)} & 6.74 & 2.40 & 1.67 & 1.19 & 0.75  \\
    & \multicolumn{1}{l}{GSM8K Acc (\%)} & 66.0 & 59.5 & 57.7 & 56.8 & 50.6 \\
    & \multicolumn{1}{l}{MATH Acc (\%)} & 20.3 & 16.5 & 15.8 & 15.9 & 13.4 \\ \bottomrule
\end{tabular}
}

\caption{Pass@1 accuracy and model size of
8-bit-quantized Llama 2-7B models compressed by
FLAP, AFM, and IMPACT for mathematical reasoning. The results for 8-bit FLAP are taken from \citet{basel}.
}
\label{tbl:7b-flap-impact}

\end{table}

 \begin{table}[!t]
\centering
\resizebox{\linewidth}{!}{
\begin{tabular}{llrrrrr}
\toprule
\multirow{3}{*}{8-bit FLAP} 
    & \multicolumn{1}{l}{Model Size (GB)} & 13.02 & 10.41 & 9.11 & 7.81 & 6.51  \\ 
    & \multicolumn{1}{l}{GSM8K Acc (\%)} & 72.7 & 59.1 & 53.2 & 16.0 & 0.0 \\
    & \multicolumn{1}{l}{MATH Acc (\%)} & 21.8 & 13.9 & 9.9 & 5.6 & 0.0  \\ \midrule
\multirow{3}{*}{8-bit AFM} 
    & \multicolumn{1}{l}{Model Size (GB)} & 13.02 & 5.33 & 3.81 & 2.59 & 1.63   \\ 
    & \multicolumn{1}{l}{GSM8K Acc (\%)} & 72.7 & 63.8 & 55.2 & 54.8 & 46.1 \\
    & \multicolumn{1}{l}{MATH Acc (\%)} & 21.8 & 17.5 & 13.9 & 12.0 & 10.6  \\ \midrule
\multirow{3}{*}{8-bit IMPACT} 
    & \multicolumn{1}{l}{Model Size (GB)} & 13.02 & 4.87 & 3.58 & 2.28 & 1.49 \\ 
    & \multicolumn{1}{l}{GSM8K Acc (\%)} & 72.7 & 65.4 & 61.8 & 58.5 & 50.9 \\
    & \multicolumn{1}{l}{MATH Acc (\%)} & 21.8 & 17.6 & 16.9 & 16.2 & 14.2  \\ \bottomrule
\end{tabular}
}
\caption{Pass@1 accuracy and model size of
8-bit-quantized Llama 2-13B models compressed by
FLAP, AFM, and IMPACT for mathematical reasoning.}
\label{tbl:13b-flap-impact}
\end{table}

\begin{table}[!t]
\centering
\resizebox{\linewidth}{!}{
\begin{tabular}{cccc}
\toprule
Model Variant & Model Size (GB) & Task & Accuracy (\%) \\
\midrule
\multirow{2}{*}{4-bit-quantized 7B} 
    & \multirow{2}{*}{3.37} & GSM8K & 39.2 \\
    &                       & MATH  & 8.8  \\
\midrule
\multirow{2}{*}{4-bit-QLoRA 7B} 
    & \multirow{2}{*}{3.45} & GSM8K & 54.1 \\
    &                       & MATH  & 12.6 \\
\midrule
\multirow{2}{*}{4-bit-quantized 13B} 
    & \multirow{2}{*}{6.51} & GSM8K & 52.8 \\
    &                       & MATH  & 12.5 \\
\midrule
\multirow{2}{*}{4-bit-QLoRA 13B} 
    & \multirow{2}{*}{6.63} & GSM8K & 61.0 \\
    &                       & MATH  & 14.3 \\
\bottomrule
\end{tabular}
}
\caption{Pass@1 accuracy and model size of Llama 2-7B and -13B quantized using standard 4-bit quantization and 4-bit QLoRA on the mathematical reasoning task. Part of the results is taken from \citet{basel}.
}
\label{tbl:qlora-vs-full-model-quantization-13b-and-7b}
\end{table}

\subsection{Inference Performance\label{subsec:inference}}

To evaluate inference performance, we measure the throughput and memory usage of compressed models on the mathematical reasoning task. Figure~\ref{fig:inference} presents results for models compressed using SVD, FWSVD, AFM, ASVD, and IMPACT across various model sizes. As expected, larger models exhibit lower throughput and higher memory consumption across all methods. When model size is held constant, all approaches demonstrate comparable throughput and memory consumption. However, because IMPACT achieves similar accuracy at smaller model sizes, it delivers higher throughput and lower memory usage at equivalent accuracy levels. In particular, compared to AFM—the strongest baseline—IMPACT improves throughput by up to 41\% and reduces memory use by up to 42\%.

\begin{figure}[t]

\centering
        \includegraphics[width=1\linewidth]{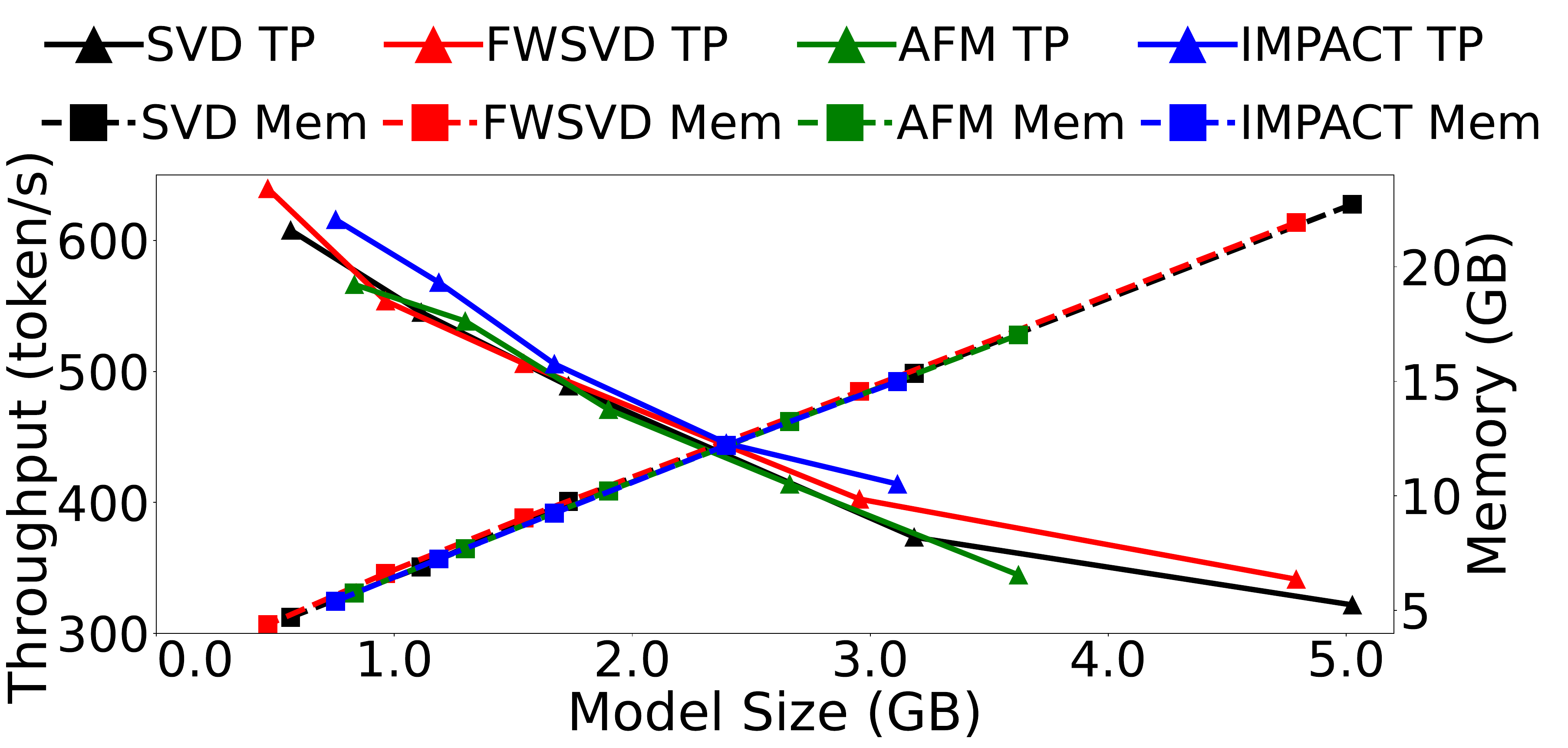}
        \vspace{-15pt}
        \caption{Throughput and memory consumption of compressed models. Exact values are provided in Table \ref{tbl:inference-results}.}
    \label{fig:inference}
\end{figure}

\section{Conclusion\label{sec:conclusion}}

This paper introduces IMPACT, a principled framework for low-rank model compression that explicitly links activation reconstruction to model performance. In contrast to prior methods that either compress weights directly or minimize activation reconstruction error, IMPACT guides activation reconstruction along directions most critical to model behavior. By formulating and solving a well-grounded optimization problem, we derive a closed-form solution in which the optimal reconstruction bases are the eigenvectors of an importance-weighted activation covariance matrix.

Our empirical results across multiple LLMs and multiple benchmarks demonstrate that IMPACT consistently achieves greater compression—up to 55.4\% more than the state-of-the-art—while maintaining similar accuracy. These findings not only validate the theoretical underpinnings of our method but also highlight its practical effectiveness for real-world deployment.

IMPACT offers a general and extensible foundation for future compression research. By establishing a formal link between compression decisions and performance outcomes, our work provides both insight and actionable tools for efficient LLM deployment—advancing the broader goal of making powerful models more accessible and sustainable.

\section*{Limitations\label{sec:limitation}}
This work establishes a connection between activation reconstruction and model performance, enabling performance-aware activation space reconstruction for model compression. Model compression is a crucial topic in natural language processing, influencing the deployment and applicability of NLP models. We believe that our approach, which links model design decisions to performance, has the potential to extend beyond model compression and be applied to other areas. We leave exploration of this broader applicability to future work.
\section*{Acknowledgments \label{sec:ack}}
This work is supported by a faculty startup grant from Iowa State University. Computational resources are provided by the HPC of the university, including equipment funded by NSF under MRI Grant Nos.~1726447 and 2018594.
\bibliography{custom}
\appendix
\setcounter{theorem}{0}
\section{Theoretical Derivation}
\label{sec:appendix}

\subsection{Mathematical Preliminaries}

We first introduce key mathematical definitions and properties that serve as the foundation for our derivation.

\begin{definition}[Differentiation Convention]
\label{def:differentiation}
For a differentiable function $\ell(\mathbf{y)}: \mathbb{R}^n \rightarrow \mathbb{R}$ where $\mathbf{y} = [y_1,\dots, y_n]^\top \in \mathbb{R}^n$, the derivative of $\ell$ with respect to $\mathbf{y}$ following the denominator-layout convention is given by the row vector

{
\small \[
  \dfrac{\partial\ell}{\partial\mathbf y}
  \;=\;
  \begin{bmatrix}
    \displaystyle\frac{\partial\ell}{\partial y_1} &
    \dots &
    \displaystyle\frac{\partial\ell}{\partial y_n}
  \end{bmatrix}
\]}
\end{definition}
We maintain this convention throughout our derivation. 

\begin{definition}[Hadamard Product]
\label{def:hadamard}
Given two vectors \( \mathbf{p} = [p_1,\dots, p_n]^\top \in \mathbb{R}^n \) and \( \mathbf{q} = [q_1,\dots, q_n]^\top \in \mathbb{R}^n \), the Hadamard product (element-wise product) is defined as  

{
\small
\[
\mathbf{p} \odot \mathbf{q} =
\begin{bmatrix}
p_1 q_1 \\
\vdots \\
p_n q_n
\end{bmatrix} \in \mathbb{R}^n
\]}
\end{definition}

\begin{definition}[Orthogonality and Normalization]
\label{def:orthogonality_normalization}
Given column vectors \( \mathbf{u}_i, \mathbf{u}_j \in \mathbb{R}^{n} \), their orthogonality and normalization properties are defined as follows:
\begin{itemize}
    \item Vectors \(\mathbf{u}_i\) and \( \mathbf{u}_j \) are orthogonal if their inner product satisfies
    
    {\vspace{-6pt} 
    \small \[
    \mathbf{u}_i^\top \mathbf{u}_j = 0
    \]}
    
    \item A vector \( \mathbf{u}_i \) is normalized if 
    
    {\vspace{-6pt} 
    \small\[
    \mathbf{u}_i^\top \mathbf{u}_i = \| \mathbf{u}_i \|^2 = 1
    \]}
\end{itemize}
\end{definition}

\begin{property}[QM-AM Inequality]
\label{def:am_qm}
For the set $\{y_i \; | \; y_i \in \mathbb{R}_{\ge 0}, \;i = 1, \cdots, n \}$, the arithmetic mean (AM) and the quadratic mean (QM) are defined as: 

{\vspace{-6pt}\small
\[
\mathrm{AM} = \frac{1}{n} \sum_{i=1}^{n} y_i,
\qquad
\mathrm{QM} = \sqrt{\frac{1}{n} \sum_{i=1}^{n} y_i^{2}}
\]}%
The QM-AM inequality states that:
$\mathrm{QM} \geq \mathrm{AM}$, \text{\rev{with equality} if and only if } $y_1 = \dots = y_n$.
\end{property}

\begin{method}[Lagrange Multiplier Method]
\label{def:lagrange}
The Lagrange multiplier method determines the local extrema of a function under explicit functional constraints. Given an objective function \( f: \mathbb{R}^n \to \mathbb{R} \) and a constraint function \( g: \mathbb{R}^n \to \mathbb{R} \), where the constraint is given by $g(x) = 0$, the Lagrangian function is defined as: 

{
\small \[
\mathcal{L}(x, \lambda) = f(x) + \lambda g(x)
\]}%
where \( \lambda \in \mathbb{R} \) is the Lagrange multiplier. The optimal solution is obtained by solving the system of equations:

{
\small \[
\frac{d}{dx} \mathcal{L}(x, \lambda) = 0, \quad \frac{d}{d\lambda} \mathcal{L}(x, \lambda) = 0
\]}
\end{method}



\subsection{Bounding Theorem\label{appendix:A2}}

\begin{theorem}
\label{theorem:ap_theorem1}
Suppose the loss function $\ell$ is $C^1$-smooth and the activation dimension is $d$. Then the objective function 

{\vspace{-7pt} 
\small
\begin{equation}
    f(\{\mathbf{u}_k\}) = \alpha \mathbb{E}\!\left[\|\mathbf{y} - \mathbf{\hat{y}} \|^2\right] 
    + \beta \mathbb{E}\! \left[ (\ell(\mathbf{y}) - \ell(\mathbf{\hat{y}}))^2 \right]
    \label{eq:ap_eqn11}
\end{equation}
}%
is upper bounded by:

{\vspace{-6pt} 
\small
\[
f(\{\mathbf{u}_k\}) \leq  \mathbb{E}\! \left[ \left\| \sqrt{\beta d \, \mathbb{E}\! \left[ \left( \frac{\partial \ell}{\partial \mathbf{y}} \right)^2 \right]^\top + \alpha }  \odot (\mathbf{y} - \mathbf{\hat{y}}) \right\|^2 \right]
\]
}
\end{theorem}

\begin{proof}
Performing Taylor expansion\footnotemark[1] on the loss function, we obtain:

\footnotetext[1]{
First-order Taylor approximations of the loss are widely used to estimate the effect of small parameter or structural perturbations in neural networks, due to their simplicity and empirical performance. This formulation underlies many modern sensitivity-based compression and importance estimation methods \citep{molchanov2017pruning, molchanov2019importance, lee2019snip}, and our approach follows this established first-order practice.
}

{
\small \[
    \ell(\mathbf{\hat{y}}) \approx \ell(\mathbf{y}) + \frac{\partial \ell}{\partial \mathbf{y}} (\mathbf{\hat{y}} - \mathbf{y}) 
\]}%
The higher-order terms (e.g., second-order terms and beyond) are ignored because they are computationally expensive to estimate and difficult to capture accurately in practical applications. 
Plugging this into Equation~\eqref{eq:ap_eqn11}, we get the following results:
{\begin{equation*}
    \scalebox{0.72}
    {$
    \begin{aligned}
        f(\{\mathbf{u}_k\})& \approx \alpha \mathbb{E}\!\left[\|\mathbf{y} - \mathbf{\hat{y}} \|^2\right] 
        + \beta \mathbb{E}\! \left[ \left( \frac{\partial \ell}{\partial \mathbf{y}} (\mathbf{y} - \mathbf{\hat{y}}) \right)^2 \right] \\
        &= \alpha \mathbb{E}\!\left[\|\mathbf{y} - \mathbf{\hat{y}} \|^2\right] 
        + \beta d^2 \,\mathbb{E}\! \left[\left( \frac{\frac{\partial \ell}{\partial \mathbf{y}} (\mathbf{y} - \mathbf{\hat{y}}) }{d}\right)^2  \right] \\
        & = \alpha \mathbb{E}\!\left[\|\mathbf{y} - \mathbf{\hat{y}} \|^2\right] 
        + \beta d^2 \,\mathbb{E}\! \left[\left( \frac{\sum_{i=1}^d \left(\frac{\partial \ell}{\partial \mathbf{y}}\right)_i (\mathbf{y} - \mathbf{\hat{y}})_i }{d}\right)^2  \right] \\
        & \le \alpha \mathbb{E}\!\left[\|\mathbf{y} - \mathbf{\hat{y}} \|^2\right] 
        + \beta d^2 \,\mathbb{E}\! \left[\left( \frac{\sum_{i=1}^d \left|\left(\frac{\partial \ell}{\partial \mathbf{y}}\right)_i (\mathbf{y} - \mathbf{\hat{y}})_i \right|}{d}\right)^2  \right]  
    \end{aligned}
    $}
    \label{eq:ap_eqn3}
\end{equation*}
}%
Here, the subscript \(i\) denotes the \(i\)-th element of a vector; for instance, \((\mathbf{y - \hat{y}})_i\) refers to the \(i\)-th element of \((\mathbf{y - \hat{y}})\), and $y_i$ and $\hat{y_i}$ are the \(i\)-th elements of $\mathbf{y}$ and $\mathbf{\hat{y}}$, respectively. The second term corresponds to the square of the arithmetic mean of $d$ subterms, which can be upper bounded by the square of their quadratic mean.  By applying the QM–AM inequality, we obtain:
{\vspace{0pt}
\begin{equation*}
    \scalebox{0.73}
    {$
    \begin{aligned}
        f(\{\mathbf{u}_k\}) 
        & \le \alpha \mathbb{E}\!\left[\|\mathbf{y} - \mathbf{\hat{y}} \|^2\right] 
        + \beta d^2 \,\mathbb{E}\! \left[ \frac{\sum_{i=1}^d \left(\frac{\partial \ell}{\partial \mathbf{y}}\right)_i^2 (\mathbf{y} - \mathbf{\hat{y}})_i^2 }{d}  \right] \\
        & = \alpha \mathbb{E}\!\left[\|\mathbf{y} - \mathbf{\hat{y}} \|^2\right] 
        + \beta d \,\mathbb{E}\! \left[ { \left(\frac{\partial \ell}{\partial \mathbf{y}}\right)^2 (\mathbf{y} - \mathbf{\hat{y}})^2 }  \right] \\
        & \approx \alpha \mathbb{E}\!\left[\|\mathbf{y} - \mathbf{\hat{y}} \|^2\right] 
        + \beta d \, { \mathbb{E}\!\left[\left(\frac{\partial \ell}{\partial \mathbf{y}}\right)^2\right] \mathbb{E}\! \left[(\mathbf{y} - \mathbf{\hat{y}})^2\right] } \\
        & = \alpha \mathbf{1}_d^\top \mathbb{E}\!\left[(\mathbf{y} - \hat{\mathbf{y}})^2\right]
        + \beta d \, { \mathbb{E}\!\left[\left(\frac{\partial \ell}{\partial \mathbf{y}}\right)^2\right] \mathbb{E}\! \left[(\mathbf{y} - \mathbf{\hat{y}})^2\right] } \\
        & = \left(\alpha \mathbf{1}_d^\top 
        + \beta d \,  \mathbb{E}\!\left[\left(\frac{\partial \ell}{\partial \mathbf{y}}\right)^2\right] \right) \mathbb{E}\! \left[(\mathbf{y} - \mathbf{\hat{y}})^2\right]  \\
        & = \mathbb{E}\!\left[ \left(\alpha + \beta d \, \mathbb{E} \left[ \left(\frac{\partial \ell}{\partial \mathbf{y}}\right)^2 \right] \right) (\mathbf{y} - \mathbf{\hat{y}})^2\right] \\
        & = \mathbb{E}\!\left[ \left(\sqrt{\beta d \, \mathbb{E} \left[ \left(\frac{\partial \ell}{\partial \mathbf{y}}\right)^2 \right] + \alpha }\right )^2 (\mathbf{y} - \mathbf{\hat{y}})^2\right] \\
        & = \mathbb{E}\!\left[ \sum_{i=1}^d \left( \left( \sqrt{\beta d \, \mathbb{E} \left[ \left(\frac{\partial \ell}{\partial \mathbf{y}}\right)^2 \right] + \alpha } \right)_i (\mathbf{y} - \mathbf{\hat{y}})_i\right)^2\right]
    \end{aligned}
    $}
    \label{eq:ap_eqn30}
\end{equation*}
}%
Finally, using the Hadamard product (Definition~\ref{def:hadamard}), we get:

{
\vspace{-8pt}
\small
\begin{equation*}
    \begin{aligned}
        f(\{\mathbf{u}_k\}) \leq  \mathbb{E}\! \left[ \left\| \sqrt{\beta d\; \mathbb{E}\! \left[ \left( \frac{\partial \ell}{\partial \mathbf{y}} \right)^2 \right]^\top + \alpha }  \odot (\mathbf{y} - \mathbf{\hat{y}}) \right\|^2 \right]
    \end{aligned}
    \label{eq:ap_eqn5}
\end{equation*}
}
\end{proof}
\subsection{Activation Space Transformation Theorem\label{appendix:A3}}
\begin{theorem}
\label{theorem:ap_theorem2}
Applying the projection condition 

{\vspace{-6pt} 
\small
\[
\mathbf{a} \odot \left( \mathbf{\hat{y}} - \mathbb{E}\!\left[\mathbf{y}\right] \right) = \left( \sum_k \mathbf{u}_k \mathbf{u}_k^\top \right) \left(\mathbf{a} \odot \left( \mathbf{y} - \mathbb{E}\!\left[\mathbf{y}\right] \right)\right)
\]
}%
where the transformation coefficient $\mathbf{a}$ is 

{\small
\[
\mathbf{a} = \sqrt{(1 - \eta) \frac{\mathbb{E}\! \left[ \left( \frac{\partial \ell}{\partial \mathbf{y}} \right)^2 \right]^\top }{\frac{1}{d} \mathbb{E}\left[ \left\| \frac{\partial \ell}{\partial \mathbf{y}} \right\|^2 \right]} + \eta }
\]
}%
and utilizing the activation transformation  \( \mathbf{\tilde{y}} = \mathbf{a} \odot \left(\mathbf{y} - \mathbb{E}\!\left[\mathbf{y}\right]\right) \), the objective 

{\vspace{-3pt}
\small
    \begin{equation*}
    \scalebox{0.96}{$
        h(\{\mathbf{u}_k\}) = \mathbb{E}\! \Bigg[ \Bigg\| \sqrt{
        \frac{(1 - \eta)}{\frac{1}{d}\mathbb{E}\! \left[\left\| \frac{\partial \ell}{\partial \mathbf{y}} \right\|^2\right]} 
        \mathbb{E}\! \left[ \left( \frac{\partial \ell}{\partial \mathbf{}
        y} \right)^2 \right]^\top + \eta }  
        \odot (\mathbf{y} - \mathbf{\hat{y}}) \Bigg\|^2 \Bigg]
    \label{eq:ap_eqn7}
    $}
\end{equation*}
}%
becomes:

{\vspace{-3pt}
\small
\begin{equation}
    h(\{\mathbf{u}_k\}) = \mathbb{E}\! \left[ \mathbf{\tilde{y}}^\top \left( \mathbf{I} - \sum_k \mathbf{u}_k \mathbf{u}_k^\top \right) \mathbf{\tilde{y}} \right]
    \label{eq:ap_eqn13}
\end{equation}}
\end{theorem}
\begin{proof}
Using the transformation coefficient $\mathbf{a}$, the upper bound function can be written as

{
\small \begin{equation}
    \label{eq:ap_eqn14}
    h(\{\mathbf{u}_k\}) = \mathbb{E}\! \left[ \left\|  \mathbf{a} \odot (\mathbf{y} - \mathbf{\hat{y}}) \right\|^2 \right]
\end{equation}}%
Given the projection condition

{\vspace{-11pt}
\small
\begin{equation*}
    \mathbf{a} \odot \left( \mathbf{\hat{y}} - \mathbb{E}\!\left[\mathbf{y}\right] \right) = \left( \sum_k \mathbf{u}_k \mathbf{u}_k^\top \right) \left(\mathbf{a} \odot \left( \mathbf{y} - \mathbb{E}\!\left[\mathbf{y}\right] \right)\right)
\end{equation*}
}%
subtracting {\small $\mathbf{a} \odot \left( \mathbf{y} - \mathbb{E}\!\left[\mathbf{y}\right] \right)$} on both sides and, after that, multiplying both sides with -1, we have 

{\vspace{-8pt} 
\small
\begin{equation}
    \mathbf{a} \odot (\mathbf{y} - \mathbf{\hat{y}}) = \left( \mathbf{I} - \sum_k \mathbf{u}_k \mathbf{u}_k^\top \right) \left(\mathbf{a}\odot (\mathbf{y} - \mathbb{E}\!\left[\mathbf{y}\right])\right)
    \label{eq:ap_eqn15}
\end{equation}
}%
Combining Equations~\eqref{eq:ap_eqn14} and~\eqref{eq:ap_eqn15}, we obtain:

{\vspace{-8pt} 
\small
\begin{equation*}
    h(\{\mathbf{u}_k\}) = \mathbb{E}\! \left[ \left\| \left( \mathbf{I} - \sum_k \mathbf{u}_k \mathbf{u}_k^\top \right) (\mathbf{a} \odot (\mathbf{y} - \mathbb{E}\!\left[\mathbf{y}\right])) \right\|^2 \right]
    \label{eq:ap_eqn9}
\end{equation*}
}%

\noindent Given the transformed activation {\small $\mathbf{\tilde{y}} = \mathbf{a} \odot (\mathbf{y} - \mathbb{E}\!\left[\mathbf{y}\right])$}, the objective function can be rewritten as:

{
\vspace{-9pt}
\small
\begin{equation*}
    \begin{aligned}
        h&(\{\mathbf{u}_k\}) = \mathbb{E}\! \left[ \left\| \left( \mathbf{I} - \sum_k \mathbf{u}_k \mathbf{u}_k^ \top \right) \mathbf{\tilde{y}} \right\|^2 \right] \\
        &= \mathbb{E} \Big[ \mathbf{\tilde{y}}^ \top  \left( \mathbf{I} - \sum_k \mathbf{u}_k \mathbf{u}_k^\top \right) \cdot \left( \mathbf{I} - \sum_k \mathbf{u}_k \mathbf{u}_k^ \top \right) \mathbf{\tilde{y}} \Big] \\
        &= \mathbb{E} \Big[ \mathbf{\tilde{y}}^ \top  \left( \mathbf{I} - 2\sum_k \mathbf{u}_k \mathbf{u}_k^ \top  + \sum_{i,j} \mathbf{u}_i \mathbf{u}_i^\top \mathbf{u}_j \mathbf{u}_j^\top\right) \mathbf{\tilde{y}} \Big] 
    \end{aligned}
\end{equation*}
}%
As $\{\mathbf{u}_k\}$ are orthogonal vectors where  $\mathbf{u}_i\mathbf{u}_j^\top = 0$ if $i \neq j$, we obtain:

{\vspace{-8pt}
\small
\begin{equation*}
    \begin{aligned}
        h(\{\mathbf{u}_k\}) &= \mathbb{E} \Big[ \mathbf{\tilde{y}}^ \top  \left( \mathbf{I} - 2\sum_k \mathbf{u}_k \mathbf{u}_k^ \top  + \sum_k \mathbf{u}_k \mathbf{u}_k^ \top\right) \mathbf{\tilde{y}} \Big] \\
        & = \mathbb{E} \Big[ \mathbf{\tilde{y}}^ \top  \left( \mathbf{I} - \sum_k \mathbf{u}_k \mathbf{u}_k^ \top \right) \mathbf{\tilde{y}} \Big] 
    \label{eq:ap_eqn10}
    \end{aligned}
\end{equation*}
}
\end{proof}
\subsection{Weighted Covariance Matrix\label{appendix:A4}}
\begin{theorem}
    \label{theorem:ap_theorem3}
    The importance-weighted activation covariance matrix $\mathbf{C}$, given by $\mathbf{C} = \mathbb{E}\!\left[\mathbf{\tilde{y}}\mathbf{\tilde{y}}^\top\right]$, is equal to the Hadamard product of the activation covariance matrix $\mathrm{Cov}\!\left(\mathbf{y}\right)$ and the gradient-informed importance matrix $\mathbf{M}$, i.e., 
    
    {\small \begin{equation*}
        \mathbf{C} =  \mathrm{Cov}\!\left(\mathbf{y}\right) \odot \mathbf{M}
    \end{equation*}}%
    where
    
    {\vspace{-4pt} 
    \small \[
    \mathrm{Cov}(\mathbf{y}) = \mathbb{E}\! \left[ (\mathbf{y} - \mathbb{E}\!\left[\mathbf{y}\right]) (\mathbf{y} - \mathbb{E}\!\left[\mathbf{y}\right])^\top \right]
    \]}
    {\small \[
            \mathbf{M} = \mathbf{a} \mathbf{a}^\top
    \]}
\end{theorem}
\begin{proof}
    As the importance-weighted activation covariance matrix $\mathbf{C}$ is given by {\small $\mathbf{C} = \mathbb{E}\!\left[\mathbf{\tilde{y}}\mathbf{\tilde{y}}^\top\right]$}, plugging the activation transformation {\small \( \mathbf{\tilde{y}} = \mathbf{a} \odot \left(\mathbf{y} - \mathbb{E}\!\left[\mathbf{y}\right]\right) \)} into this expression, matrix $\mathbf{C}$ can be written as: 
    
    {\vspace{-7pt}
    \small \begin{equation*}
        \mathbf{C} = \mathbb{E}\! \left[ \left( \mathbf{a} \odot (\mathbf{y} - \mathbb{E}\!\left[\mathbf{y}\right]) \right) \left( \mathbf{a} \odot (\mathbf{y} - \mathbb{E}\!\left[\mathbf{y}\right]) \right)^\top \right]
    \end{equation*}}%
    The $(i,j)^\mathrm{th}$ element of $\mathbf{C}$ is:
    
    {\small \[
    \mathbf{C}_{ij} = \mathbb{E}\! \left[ a_i (y_i - \mathbb{E}\!\left[y_i\right]) \cdot a_j (y_j - \mathbb{E}\!\left[y_j\right]) \right]
    \]}%
    where $a_i$, $a_j$ and $y_i$ and $y_j$ are the $i^\mathrm{th}$ and $j^\mathrm{th}$ elements of $\mathbf{a}$ and $\mathbf{y}$, respectively. Since $\mathbf{a}$ is a deterministic vector, the expectation becomes:
    
    {\small \[
    \mathbf{C}_{ij} = a_i a_j \mathbb{E}\! \left[ (y_i - \mathbb{E}\!\left[y_i\right]) \cdot (y_j - \mathbb{E}\!\left[y_j\right]) \right]
    \]}%
    The term {\small $\mathbb{E}\! \left[ (y_i - \mathbb{E}\!\left[y_i\right]) \cdot (y_j - \mathbb{E}\!\left[y_j\right]) \right]$} is the $(i,j)^\mathrm{th}$ element of the covariance matrix $\mathrm{Cov}(\mathbf{y})$, which is given by:
    
    {
    \small \[
    \mathrm{Cov}(\mathbf{y}) = \mathbb{E}\! \left[ (\mathbf{y} - \mathbb{E}\!\left[\mathbf{y}\right]) (\mathbf{y} - \mathbb{E}\!\left[\mathbf{y}\right])^\top \right]
    \]}%
    Thus, 
    
    {\vspace{-4pt} 
    \small \[
    \mathbf{C}_{ij} = a_i a_j \mathrm{Cov}(\mathbf{y})_{ij}
    \]}%
    For the gradient-informed importance matrix $\mathbf{M}$, which is defined as $\mathbf{M} = \mathbf{a} \mathbf{a}^\top$, we have $\mathbf{M}_{ij} = a_i a_j$. Hence, the $(i,j)^\mathrm{th}$ element of matrix $\mathbf{C}$ can be expressed as 
    
    {\vspace{-1pt}
    \small \[
    \mathbf{C}_{ij} = \mathbf{M}_{ij} \mathrm{Cov}(\mathbf{y})_{ij}
    \]}%
    Therefore, the importance-weighted activation covariance matrix $\mathbf{C}$ is the Hadamard product of the covariance matrix $\mathrm{Cov}(\mathbf{y})$ and the gradient-informed importance matrix $\mathbf{M}$:
    
    {\small \[
    \mathbf{C} = \mathrm{Cov}(\mathbf{y}) \odot \mathbf{M}
    \]}
    \end{proof}
    \begin{corollary}
        The importance-weighted activation covariance matrix $\mathbf{C}$ is positive semidefinite and symmetric.
        \label{corollary:col1}
    \end{corollary}
    \begin{proof}
        The gradient-informed importance matrix $\mathbf{M}$, which is given by $\mathbf{M} = \mathbf{a} \mathbf{a}^\top$, is positive semidefinite and symmetric. Similarly, the covariance matrix $\mathrm{Cov}(\mathbf{y})$, which is given by
        
        {\small \[
            \mathrm{Cov}(\mathbf{y}) = \mathbb{E}\! \left[ (\mathbf{y} - \mathbb{E}\!\left[\mathbf{y}\right]) (\mathbf{y} - \mathbb{E}\!\left[\mathbf{y}\right])^\top \right],
        \]}%
        is also positive semidefinite and symmetric. According to the Schur product theorem~\citep{zhang2006schur}, the Hadamard product of two positive semidefinite matrices is also positive semidefinite. Therefore, the importance-weighted activation covariance matrix $\mathbf{C} = \mathbf{M} \odot \mathrm{Cov}(\mathbf{y})$ is positive semidefinite and symmetric. 
    \end{proof}

\subsection{Reconstruction Direction Theorem\label{appendix:A5}}
\begin{theorem}
    \label{theorem:ap_theorem4}
    To minimize the objective \(h(\{\mathbf{u}_k\})\) under orthonormality 
constraints, the optimal \(k^\mathrm{th}\) 
reconstruction direction \(\mathbf{u}_k\) is the eigenvector corresponding to the \(k^\mathrm{th}\) 
largest eigenvalue of the importance-weighted activation covariance matrix \(\mathbf{C}\).
\end{theorem}
\begin{proof}
Taking the partial derivative of the Lagrangian function with respect to each reconstruction direction $\mathbf{u}_k$ and setting it to zero yields the optimality condition:

{ \small \[
\frac{\partial L}{\partial \mathbf{u}_k} = - 2 \mathbf{u}_k^\top \mathbf{C} + 2 \lambda_k \mathbf{u}_k^\top = 0
\]}%
Rearranging this equation and taking the transpose of both sides, we obtain:

{ \small \[
\mathbf{C}^\top \mathbf{u}_k = \mathbf{u}_k \lambda_k
\]}%
Since the matrix $\mathbf{C}$ is symmetric (as established in Corollary~\ref{corollary:col1}), we have $\mathbf{C} = \mathbf{C}^\top$. By substituting this property, we arrive at:

{ \small \begin{equation}
    \mathbf{C} \mathbf{u}_k= \lambda_k \mathbf{u}_k 
    \label{eq:ap_eqn16}
\end{equation}}%
From the Equation~\eqref{eq:ap_eqn13}, we get: 

{\vspace{-10pt} 
\small
\begin{equation*}
    \begin{aligned}
        h(\{\mathbf{u}_k\}) &=  \mathbb{E}\! \left[ \mathbf{\tilde{y}}^ \top \mathbf{\tilde{y}} \right]   - \mathbb{E}\! \left[ \mathbf{\tilde{y}}^ \top\sum_k \mathbf{u}_k \mathbf{u}_k^ \top \mathbf{\tilde{y}} \right] \\
        &=  \mathbb{E}\! \left[ \mathbf{\tilde{y}}^ \top \mathbf{\tilde{y}} \right]   - \sum_k  \mathbb{E}\! \left[ \mathbf{\tilde{y}}^\top \mathbf{u}_k \mathbf{u}_k^ \top \mathbf{\tilde{y}} \right] \\
         &=  \mathbb{E}\! \left[ \mathbf{\tilde{y}}^ \top \mathbf{\tilde{y}} \right]   - \sum_k  \mathbb{E}\! \left[ \mathbf{u}_k^ \top \mathbf{\tilde{y}} \mathbf{\tilde{y}}^\top \mathbf{u}_k \right] \\
         &=  \mathbb{E}\! \left[ \mathbf{\tilde{y}}^ \top \mathbf{\tilde{y}} \right]   - \sum_k  \mathbf{u}_k^ \top \mathbb{E}\! \left[ \mathbf{\tilde{y}} \mathbf{\tilde{y}}^\top \right]  \mathbf{u}_k 
    \end{aligned}
\end{equation*}
}%
As {\small $\mathbf{C} = \mathbb{E}\!\left[\mathbf{\tilde{y}}\mathbf{\tilde{y}}^\top\right]$},

{\small \begin{equation*}
    \begin{aligned}
        h(\{\mathbf{u}_k\}) 
         &=  \mathbb{E}\! \left[ \mathbf{\tilde{y}}^ \top \mathbf{\tilde{y}} \right]   - \sum_k  \mathbf{u}_k^ \top \mathbf{C}  \mathbf{u}_k 
    \end{aligned}
\end{equation*}}%
Further based on Equation~\eqref{eq:ap_eqn16}, 

{\small \begin{equation*}
    \begin{aligned}
        h(\{\mathbf{u}_k\}) 
         &=  \mathbb{E}\! \left[ \mathbf{\tilde{y}}^ \top \mathbf{\tilde{y}} \right]   - \sum_k  \mathbf{u}_k^ \top \mathbf{u}_k \lambda_k \\ 
         &=  \mathbb{E}\! \left[ \mathbf{\tilde{y}}^ \top \mathbf{\tilde{y}} \right]   - \sum_k  \ \|\mathbf{u}_k \|^2\lambda_k 
    \end{aligned}
\end{equation*}}
As {\small $\|\mathbf{u}_k \|^2 = 1$},

{\small \begin{equation*}
    h(\{\mathbf{u}_k\}) =  \mathbb{E}\! \left[ \mathbf{\tilde{y}}^ \top \mathbf{\tilde{y}} \right]   - \sum_k \lambda_k 
\end{equation*}}%
To minimize $h(\{\mathbf{u}_k\})$, the term  $\sum_k \lambda_k $ must be maximized. Since the importance-weighted activation covariance matrix $\mathbf{C}$ is symmetric and positive semidefinite, its eigenvalues are real and non-negative. Therefore, the maximum value of $\sum_k \lambda_k $ is achieved when $\lambda_k$ is the $k^\mathrm{th}$ largest eigenvalue of $\mathbf{C}$, with $\mathbf{u}_k$ its corresponding eigenvector. 
\end{proof}

\subsection{Activation Reconstruction Theorem\label{appendix:A6}}
\begin{theorem}
    \label{theorem:ap_theorem5}
    The reconstructed activation $\mathbf{\hat{y}}$, which satisfies the projection condition
    
    {\vspace{-10pt} 
    \small \begin{equation*}
        \mathbf{a} \odot \left( \mathbf{\hat{y}} - \mathbb{E}\!\left[\mathbf{y}\right] \right) = \left( \sum_k \mathbf{u}_k \mathbf{u}_k^ \top \right) \left(\mathbf{a} \odot \left( \mathbf{y} - \mathbb{E}\!\left[\mathbf{y}\right] \right)\right),
    \vspace{-5pt}\end{equation*}}%
    is given by:
    
    {\vspace{-10pt} 
    \small
    \begin{equation*}
    \begin{aligned}
        \mathbf{\hat{y}} &=  \left[ \mathbf{U} \oslash (\mathbf{a} \cdot \mathbf{1}_r^\top) \right] \left[ (\mathbf{U} \odot (\mathbf{a} \cdot \mathbf{1}_r^\top))^\top \mathbf{W}\right] \mathbf{x} \\
         &+ \mathbb{E}[\mathbf{y}] + (\mathbf{U} \mathbf{U}^\top \odot (\frac{1}{\mathbf{a}} \cdot \mathbf{a}^\top)) (\mathbf{b} - \mathbb{E}[\mathbf{y}])
    \end{aligned}
    \end{equation*}
    }%
    where $\mathbf{U} = [\mathbf{u}_1,\dots,\mathbf{u}_r]$, $\mathbf{1}_r$ is an $r$-dimensional column vector of ones, $\mathbf{W}$ and $\mathbf{b}$ are the original layer's weight matrix and bias, $\mathbf{x}$ is the input activation, and $\oslash$ denotes element-wise (Hadamard) division.
\end{theorem}

\begin{proof}
Rearranging the projection condition

{\vspace{-12pt} 
\small \begin{equation*}
    \mathbf{a} \odot \left( \mathbf{\hat{y}} - \mathbb{E}\!\left[\mathbf{y}\right] \right) = \left( \sum_k \mathbf{u}_k \mathbf{u}_k^ \top \right) \left(\mathbf{a} \odot \left( \mathbf{y} - \mathbb{E}\!\left[\mathbf{y}\right] \right)\right),
    \vspace{-6pt}
\end{equation*}
}%
we obtain 

{\vspace{-12pt} 
\small \begin{equation*}
    \begin{aligned}
        \mathbf{a} \odot (\mathbf{\hat{y}} - \mathbb{E}\!\left[\mathbf{y}\right]) &= \sum_k \mathbf{u}_k \mathbf{u}_k^ \top  \left( \mathbf{a} \odot \left( \mathbf{y} - \mathbb{E}\!\left[\mathbf{y}\right] \right) \right)\\
        &=\sum_k \mathbf{u}_k ( \mathbf{u}_k \odot \mathbf{a})^ \top  (\mathbf{y} - \mathbb{E}\!\left[\mathbf{y}\right])) 
    \end{aligned}
\end{equation*}
}%
Applying the Hadamard division \( \oslash \mathbf{a} \) to both sides of the equation leads to:

{\vspace{-6pt} 
\small \begin{equation*}
    \mathbf{\hat{y}} - \mathbb{E}\!\left[\mathbf{y}\right] = \sum_k \mathbf{u}_k (\mathbf{u}_k \odot \mathbf{a})^\top (\mathbf{y} - \mathbb{E}\!\left[\mathbf{y}\right]) \oslash \mathbf{a}
\end{equation*}}%
Since {\small $(\mathbf{u}_k \odot \mathbf{a})^ \top$} is a row vector and {\small $(\mathbf{y} - \mathbb{E}\!\left[\mathbf{y}\right])$} is a column vector, 
{\small $(\mathbf{u}_k \odot a)^\top (\mathbf{y} - \mathbb{E}\!\left[\mathbf{y}\right])$} is a scalar, we have

{\vspace{-10pt} 
\small \[
\mathbf{\hat{y}} = \mathbb{E}\!\left[\mathbf{y}\right] + \sum_k (\mathbf{u}_k \oslash a) (\mathbf{u}_k \odot a)^\top (\mathbf{y} - \mathbb{E}\!\left[\mathbf{y}\right])
\]}%
Rewriting the equality, we obtain:

{\vspace{-8pt} 
\small \begin{equation*}
    \begin{aligned}
        \mathbf{\hat{y}} &= \left[\mathbf{u}_1\oslash \mathbf{a},\dots,\mathbf{u}_r\oslash \mathbf{a} \right]  \left[\mathbf{u}_1\odot \mathbf{a},\dots,\mathbf{u}_r\odot \mathbf{a}\right]^\top (\mathbf{y} - \mathbb{E}\!\left[\mathbf{y}\right])\\
        &\quad+\mathbb{E}\!\left[\mathbf{y}\right]\\
        &=\mathbb{E}\!\left[\mathbf{y}\right] + (\mathbf{U} \oslash (\mathbf{a} \cdot \mathbf{1}_r^ \top))(\mathbf{U} \odot (\mathbf{a} \cdot \mathbf{1}_r^ \top))^ \top (\mathbf{y} - \mathbb{E}\!\left[\mathbf{y}\right])
    \end{aligned}
\end{equation*}}%
where {\small $\mathbf{U} = [\mathbf{u}_1,\dots,\mathbf{u}_r]$}. 
Incorporating the original activation {\small $\mathbf{y} = \mathbf{W} \mathbf{x} + \mathbf{b}$}, we obtain:

{\vspace{-8pt}
\small
\[
\mathbf{\hat{y}} = \mathbb{E}\!\left[\mathbf{y}\right] + (\mathbf{U} \oslash (\mathbf{a} \cdot \mathbf{1}_r^\top))(\mathbf{U} \odot (\mathbf{a} \cdot \mathbf{1}_r^\top))^\top (\mathbf{W} \mathbf{x} + \mathbf{b} - \mathbb{E}\!\left[\mathbf{y}\right])
\]
}%
Expanding the expression, the reconstructed activation satisfies:

{
\vspace{-8pt} 
\small
\begin{equation*}
    \begin{aligned}
        \mathbf{\hat{y}} = & \left[ \mathbf{U} \oslash (\mathbf{a} \cdot \mathbf{1}_r^\top) \right] \left[ (\mathbf{U} \odot (\mathbf{a} \cdot \mathbf{1}_r^\top))^\top \mathbf{W} \right] \mathbf{x} \\
         &+ \mathbb{E}\!\left[\mathbf{y}\right] + (\mathbf{U} \oslash (\mathbf{a} \cdot \mathbf{1}_r^\top))(\mathbf{U} \odot (\mathbf{a} \cdot \mathbf{1}_r^\top))^\top (\mathbf{b} - \mathbb{E}\!\left[\mathbf{y}\right]) \\
        = & \left[ \mathbf{U} \oslash (\mathbf{a} \cdot \mathbf{1}_r^\top) \right] \left[ (\mathbf{U} \odot (\mathbf{a} \cdot \mathbf{1}_r^\top))^\top \mathbf{W} \right] \mathbf{x}\\
        &+ \mathbb{E}[\mathbf{y}] + (\mathbf{U} \mathbf{U}^\top \odot (\frac{1}{\mathbf{a}} \cdot \mathbf{a}^\top)) (\mathbf{b} - \mathbb{E}[\mathbf{y}])
    \end{aligned}
\end{equation*}    
}
\end{proof}


\section{Additional Results}
\label{sec:appendix2}

\begin{figure*}[!h]
    \centering
    \begin{subfigure}{0.45\linewidth}
        \includegraphics[width=\linewidth]{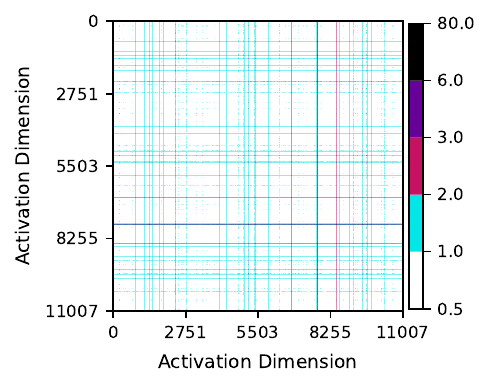}
        \caption{layers.25.mlp.up\_proj}
    \end{subfigure}
    \hfill
    \begin{subfigure}{0.45\linewidth}
        \includegraphics[width=\linewidth]{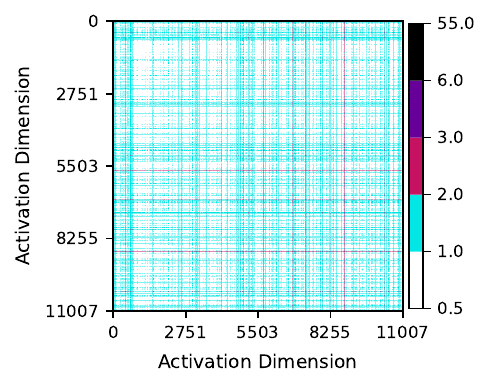}
        \caption{layers.10.mlp.up\_proj}
    \end{subfigure}
    \begin{subfigure}{0.45\linewidth}
        \includegraphics[width=\linewidth]{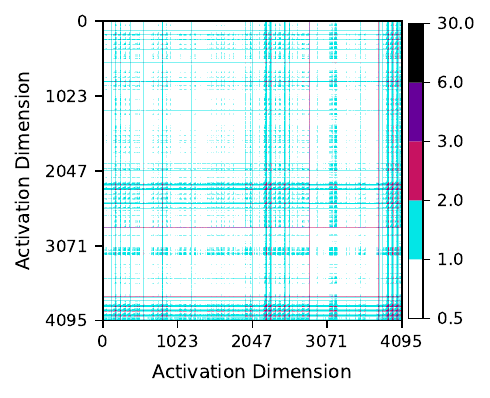}
        \caption{layers.16.self\_attn.q\_proj}
    \end{subfigure}
    \hfill
    \begin{subfigure}{0.45\linewidth}
        \includegraphics[width=\linewidth]{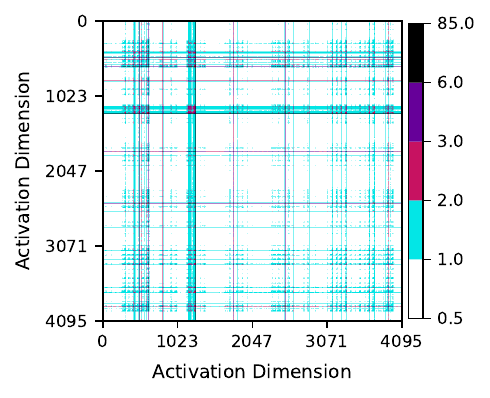}
        \caption{layers.26.self\_attn.q\_proj}
    \end{subfigure}
    \caption{
    Heatmaps of the importance matrix \( \mathbf{M} \) for four representative layers in Llama~2-7B on the mathematical reasoning task (with $\eta = 0.5$). High-intensity rows and columns indicate rare yet important activation dimensions that IMPACT identifies via gradient-based weighting and prioritizes during compression.
    }
    
    \label{fig:heatmap-M}
\end{figure*}

{\subsection{Visualization and Analysis of the Importance Matrix} \label{subsec:appendix_importance_matrix}}

\begin{figure*}[!h]
\vspace{20pt}
    \centering
    \begin{subfigure}[t]{0.48\linewidth}
        \centering
        \includegraphics[width=\linewidth]{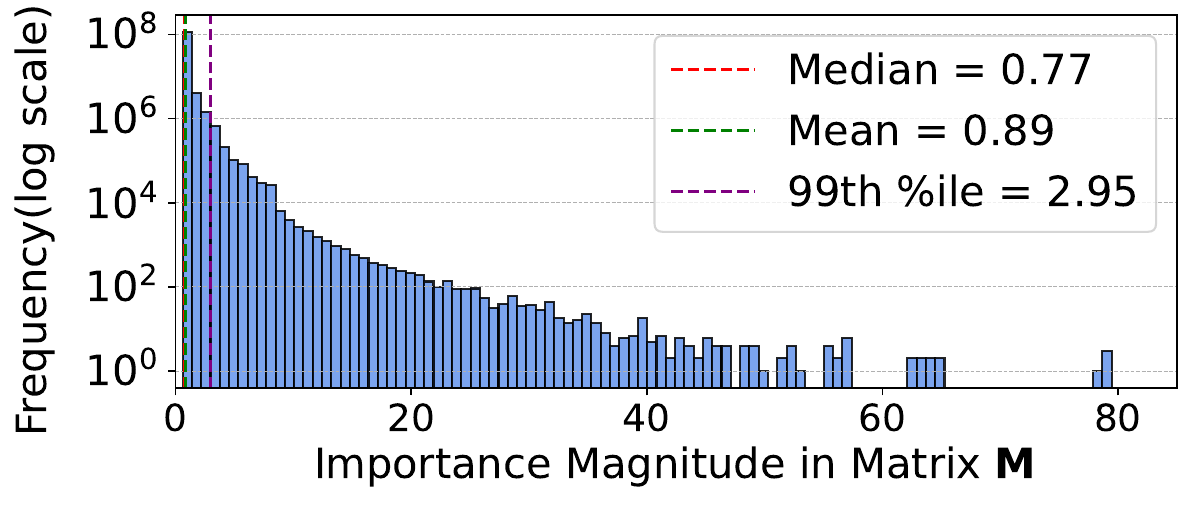}
        \caption{layers.25.mlp.up\_proj}
    \end{subfigure}
    \hfill
    \begin{subfigure}[t]{0.48\linewidth}
        \centering
        \includegraphics[width=\linewidth]{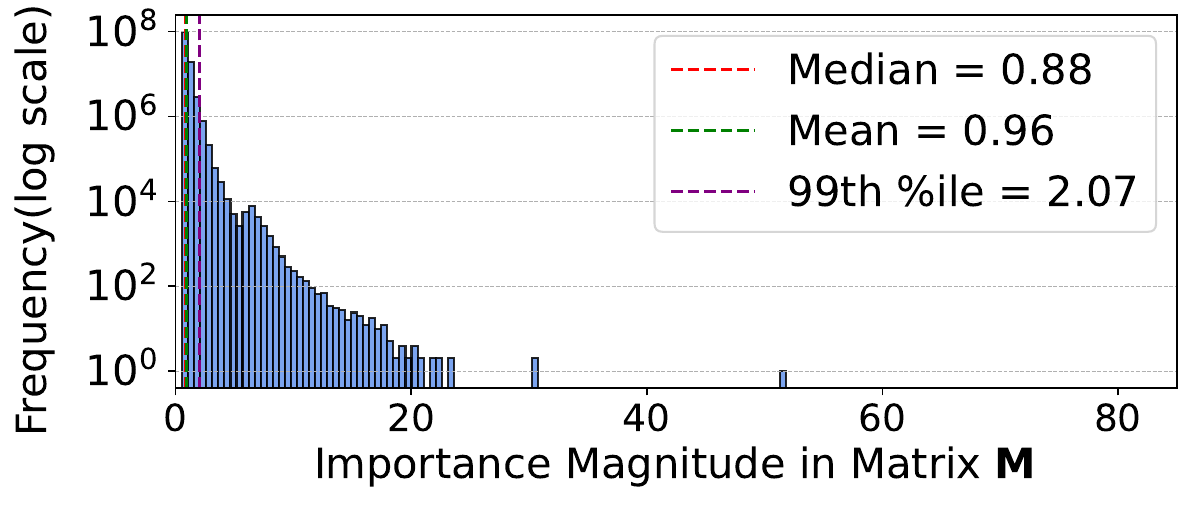}
        \caption{layers.10.mlp.up\_proj}
    \end{subfigure}
    
    \vspace{0.4em}
    
    \begin{subfigure}[t]{0.48\linewidth}
        \centering
        \includegraphics[width=\linewidth]{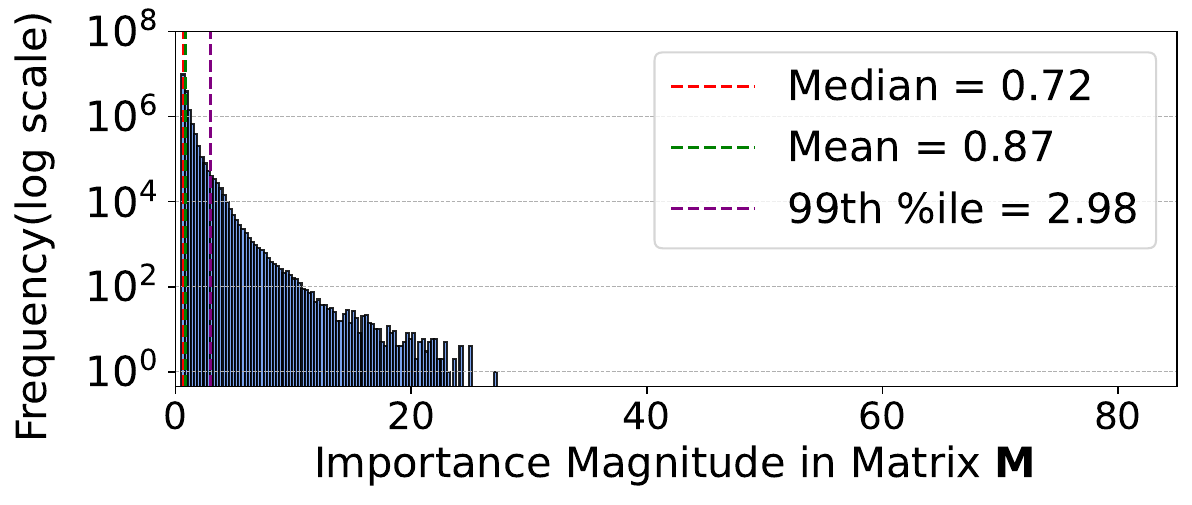}
         \caption{layers.16.self\_attn.q\_proj}
    \end{subfigure}
    \hfill
    \begin{subfigure}[t]{0.48\linewidth}
        \centering
        \includegraphics[width=\linewidth]{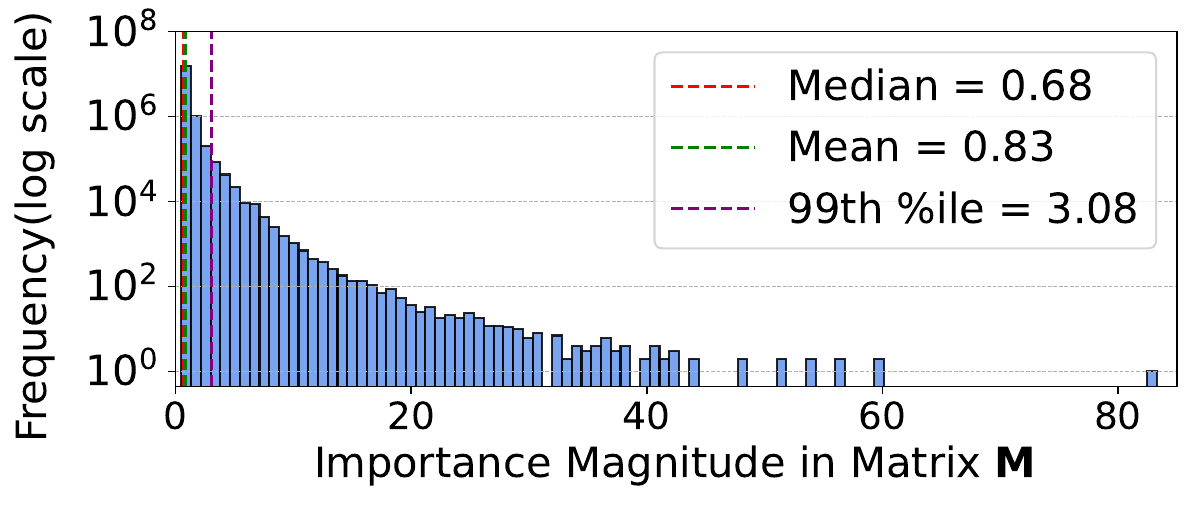}
        \caption{layers.26.self\_attn.q\_proj}
    \end{subfigure}

    \caption{Log-scale histograms of values in the importance matrix $\mathbf{M}$ (with $\eta = 0.5$) for four representative layers of Llama~2-7B on the mathematical reasoning task.}
    \label{fig:importance-distributions}
\end{figure*}

To understand how IMPACT achieves importance-aware compression, we analyze the structure and statistics of the gradient-informed importance matrix $\mathbf{M}$, offering empirical insight into its role in preserving performance-critical components.

We visualize the structure of $\mathbf{M}$ in four representative layers---two MLP and two attention---using the heatmaps in Figure~\ref{fig:heatmap-M}, where darker entries indicate greater importance for interactions between activation dimensions. Across layers, $\mathbf{M}$ displays concentrated regions of high intensity, revealing that IMPACT emphasizes a sparse subset of directions rather than treating all dimensions equally. This selective weighting is central to IMPACT’s design: by constructing the importance-weighted activation covariance matrix $\mathbf{C} = \mathrm{Cov}\!\left(\mathbf{y}\right) \odot \mathbf{M}$, the resulting eigendecomposition prioritizes performance-critical directions during compression.

Complementing the structural visualization, Figure~\ref{fig:importance-distributions} presents the distribution of values in $\mathbf{M}$ for the same layers. These histograms, plotted on a log scale, reveal highly skewed distributions. Although most entries in $\mathbf{M}$ are small, a long tail of large values persists. For instance, in "layers.25.mlp.up\_proj", the median of the entries is 0.77, the mean is 0.89, the 99th percentile reaches 2.95, and the maximum is close to 80. These values quantify the relative importance of preserving interactions between activation dimensions, as determined by their gradient sensitivity during profiling. The long-tailed distribution indicates that IMPACT selectively amplifies a small subset of critical interactions. This pattern is consistent across layers and mirrors the sparsity observed in the heatmaps, reinforcing the model’s emphasis on performance-relevant components.

{\subsection{Ablation Study on \texorpdfstring{$\eta$}{eta}}
\label{appendix:abolation}
}
We analyze the effect of the hyperparameter $\eta$, which balances the reconstruction error and gradient-informed importance weighting in Equation~\eqref{eq:eqn3}. Table~\ref{tab:eta-ablation} reports results on Llama~2-7B across different compression ratios. We vary $\eta \in \{0, 0.25, 0.5, 0.75\}$ and observe that performance remains relatively stable across this range. In particular, moderate values of $\eta$ generally achieve strong results across both GSM8K and MATH.
\begin{table}[h]
\centering
\small
\resizebox{\linewidth}{!}{
\begin{tabular}{clcccc}
\toprule
$k$ & Metric & $\eta=0.75$ & $\eta=0.5$ & $\eta=0.25$ & $\eta=0$ \\
\midrule
\multirow{2}{*}{50\%}
 & GSM8K Acc (\%) & 61.3 & 61.2 & 62.9 & 58.4 \\
 & MATH Acc (\%)  & 17.9 & 17.7 & 17.3 & 15.5 \\
\midrule
\multirow{2}{*}{60\%}
 & GSM8K Acc (\%) & 64.8 & 63.7 & 64.5 & 62.5 \\
 & MATH Acc (\%)  & 19.2 & 18.4 & 19.4 & 18.1 \\
\midrule
\multirow{2}{*}{70\%}
 & GSM8K Acc (\%) & 65.0 & 63.7 & 63.5 & 65.3 \\
 & MATH Acc (\%)  & 18.9 & 19.4 & 18.6 & 19.7 \\
\bottomrule
\end{tabular}
}
\caption{Ablation study on the hyperparameter $\eta$ under different keeping ratios $k$. Results are reported on Llama~2-7B using GSM8K and MATH.}
\label{tab:eta-ablation}
\end{table}

\begin{table}[t]
\centering
\small
\setlength{\tabcolsep}{6pt}
\begin{tabular}{lcc}
\toprule
Method & Time (Hours) & Peak GPU Memory (GB)\\
\midrule
IMPACT & 1.53 & 66.9\\
AFM    & 1.17 & 60.8\\
ASVD   & 135.83 & 28.6 \\
\bottomrule
\end{tabular}
\caption{Profiling cost comparison for Llama~2-13B across compression methods.
All the methods are run on a single NVIDIA A100 GPU.}
\label{tab:profiling_time}
\end{table}

\FloatBarrier

{\subsection{Profiling and Generalizability}
\label{app:profiling-cost}
}
\paragraph{Profiling.}
We note that the entire profiling process fits comfortably on a single GPU, requiring no multi-GPU setup. Using only 1\% of MetaMathQA~\cite{metamath} and 3\% of Code-Instructions-120K-Alpaca~\cite{codeinstructions120k} as calibration data, profiling completes in under two hours on a single NVIDIA A100 GPU for 13B models. Table~\ref{tab:profiling_time} compares the profiling cost of IMPACT and baseline methods for compressing Llama 2-13B on the mathematical reasoning task. As shown in the table, IMPACT incurs only a modest additional overhead compared to AFM. This increase in profiling cost arises from the extra backward pass required to compute gradient-based importance statistics, which slightly increases both runtime and peak memory usage. In contrast, ASVD requires significantly longer profiling time. This is because ASVD performs a binary search to determine the truncation rank for each layer, repeatedly evaluating model performance for different candidate ranks. As a result, the procedure requires a substantially larger number of forward passes, leading to orders-of-magnitude longer profiling time despite lower peak GPU memory usage.

We analyze the method's sensitivity to calibration data size. We find that using only 1\% of MetaMathQA or 3\% of Code-Instructions-120K is sufficient for convergence. IMPACT is intentionally distribution-dependent: the profiling statistics are expectations over the calibration data, designed to maximize performance on the target domain.

\vspace{-18pt}

\rev{
\paragraph{Generalizability.}
IMPACT makes no model-specific assumptions. The closed-form solution depends strictly on activation covariance and gradient-based importance signals that are generic to all Transformer architectures. As such, the framework is inherently applicable to diverse modern LLMs. 
}



\begin{table}[!t]

\centering
\resizebox{\linewidth}{!}{
\begin{tabular}
{llrrrrrrr}
\toprule
\multirow{ 3}{*}{SVD}& \multicolumn{1}{l}{Model Size (Billion)} & 6.74 & 5.03 & 3.18 & 1.73 & 1.11 & 0.56 &\\\cmidrule{2-8}
 & \multicolumn{1}{l}{GSM8K Acc (\%)} & 66.4 & 63.0 & 61.0 & 53.9 & 32.9 & 11.9 & \\
 & \multicolumn{1}{l}{MATH Acc (\%)} & 20.6 & 18.3 & 17.4 & 13.7 & 5.3 & 2.8 & \\\midrule
\multirow{ 3}{*}{FWSVD}& \multicolumn{1}{l}{Model Size (Billion)} & 6.74 & 4.79 & 2.95 & 1.54 & 0.96 & 0.47 & \\\cmidrule{2-8}
 & \multicolumn{1}{l}{GSM8K Acc (\%)} & 66.4 & 62.7 & 62.5 & 56.5 & 1.5 & 1.9 \\
 & \multicolumn{1}{l}{MATH Acc (\%)} & 20.6 & 19.2 & 17.6 & 14.2 & 1.8 & 1.5 &  \\\midrule

\multirow{ 3}{*}{ASVD}& \multicolumn{1}{l}{Model Size (Billion)} & 6.74 & 4.09 & 3.43 & 2.77 & 1.45 & 0.79 & \\\cmidrule{2-8}
 & \multicolumn{1}{l}{GSM8K Acc (\%)} & 66.4 & 61.9 & 60.9 & 59.7 & 50.5 & 29.1\\
 & \multicolumn{1}{l}{MATH Acc (\%)} & 20.6 & 18.7 & 18. 7 & 17.3 & 12.6 & 5.2 &  \\\midrule
 
 \multirow{ 3}{*}{AFM}& \multicolumn{1}{l}{Model Size (Billion)} & 6.74 & 3.62 & 2.66 & 1.90 & 1.30 & 0.83 & \\\cmidrule{2-8}
 & \multicolumn{1}{l}{GSM8K Acc (\%)} & 66.4 & 62.8 & 62.6 & 61.5 & 57.8 & 49.7 & \\
 & \multicolumn{1}{l}{MATH Acc (\%)} & 20.6 & 19.5 & 18.5 & 17.6 & 16.3 & 12.0 & \\\midrule
\multirow{ 3}{*}{IMPACT}& \multicolumn{1}{l}{Model Size (Billion)} & 6.74 & 3.11 & 2.40 & 1.67 & 1.19 & 0.75 &\\\cmidrule{2-8}
 & \multicolumn{1}{l}{GSM8K Acc (\%)} & 66.4 & 65.3 & 64.5 & 62.9 & 60.2 & 52.8 &  \\
 & \multicolumn{1}{l}{MATH Acc (\%)} & 20.6 & 19.7 & 19.4 & 17.3 & 17.2 & 14.1 &

\\\bottomrule
\end{tabular}
}

\caption{Pass@1 accuracy and model size of Llama~2-7B compressed by various algorithms for the mathematical reasoning task. The results for SVD and FWSVD are taken from \citet{basel}.\vspace{1.1em}}
\label{tbl:7b-math}
\end{table}

\begin{table}[!t]
\centering
\resizebox{\linewidth}{!}{
\begin{tabular}{llrrrrrrr} 
\toprule
\multirow{3}{*}{SVD} & \multicolumn{1}{l}{Model Size (Billion)} & 13.02 & 9.70 & 6.10 & 3.27 & 2.07 & 1.01 &\\\cmidrule{2-8}
& \multicolumn{1}{l}{GSM8K Acc (\%)} & 72.7 & 69.5 & 63.5 & 50.0 & 26.9 & 6.7 & \\
& \multicolumn{1}{l}{MATH Acc (\%)} & 22.2 & 20.8 & 17.8 & 10.8 & 5.2 & 2.2 & \\\midrule
\multirow{3}{*}{FWSVD} & \multicolumn{1}{l}{Model Size (Billion)} & 13.02 & 9.24 & 5.67 & 2.93 & 1.79 & 0.83 & \\\cmidrule{2-8}
& \multicolumn{1}{l}{GSM8K Acc (\%)} & 72.7 & 67.9 & 63.9 & 51.9 & 2.4 & 3.9 &  \\
& \multicolumn{1}{l}{MATH Acc (\%)} & 22.2 & 20.3 & 18.0 & 12.4 & 1.2 & 1.9 & \\\midrule

\multirow{3}{*}{ASVD} & \multicolumn{1}{l}{Model Size (Billion)} & 13.02 & 7.87 & 6.56 & 4.03 & 2.73 & 1.45 & \\\cmidrule{2-8}
& \multicolumn{1}{l}{GSM8K Acc (\%)} & 72.7 & 63.6 & 61.0 & 56.2 & 50.6 & 25.9 &  \\
& \multicolumn{1}{l}{MATH Acc (\%)} & 22.2 & 17.7 & 17.4 & 14.0 & 12.0 & 5.6 & \\\midrule

\multirow{3}{*}{AFM} & \multicolumn{1}{l}{Model Size (Billion)} & 13.02 & 9.63 & 5.33 & 3.81 & 2.59 & 1.63 & \\\cmidrule{2-8}
& \multicolumn{1}{l}{GSM8K Acc (\%)} & 72.7 & 70.9 & 66.4 & 65.7 & 60.8 & 47.2 &\\
& \multicolumn{1}{l}{MATH Acc (\%)} & 22.2 & 20.9 & 18.5 & 18.0 & 16.3 & 13.3 & \\\midrule
\multirow{3}{*}{IMPACT} & \multicolumn{1}{l}{Model Size (Billion)} & 13.02 & 8.66 & 4.87 & 3.58 & 2.28 & 1.49 & \\\cmidrule{2-8}
& \multicolumn{1}{l}{GSM8K Acc (\%)} & 72.7 & 70.2 & 69.3 & 65.3 & 61.2 & 54.8 & \\
& \multicolumn{1}{l}{MATH Acc (\%)} & 22.2 & 21.2 & 19.4 & 18.6 & 17.7 & 15.2 & 
\\\bottomrule
\end{tabular}
}
\caption{Pass@1 accuracy and model size of Llama~2-13B compressed by various algorithms for the mathematical reasoning task. The results for SVD and FWSVD are taken from \citet{basel}.\vspace{1.1em}}
\label{tbl:13b-math}
\end{table}



\subsection{Additional Tables}
\label{appendix:additional_figures_tables}
\rev{Tables~\ref{tbl:7b-math}--\ref{tbl:13b-prog} report the precise Pass@1 accuracy and compressed model sizes corresponding to the performance curves plotted in Figures~\ref{fig:7b-math}--\ref{fig:13b-programming}. They cover both Llama~2 (mathematical reasoning) and CodeLlama (code generation) across 7B and 13B.}

To quantify compression gains, we compare IMPACT with the smallest baseline model that attains equivalent accuracy. Detailed size reductions are reported in Tables~\ref{tbl:7b-math-size-reduction} and~\ref{tbl:13b-math-size-reduction} for Llama~2-7B and -13B, respectively, and in Tables~\ref{tbl:7b-programming-size-reduction} and~\ref{tbl:13b-programming-size-reduction} for CodeLlama-7B and -13B, respectively. Table~\ref{tbl:inference-results} reports the inference throughput and memory consumption of the compressed models.\vspace{1em}

\begin{table}[!h]
\centering
\resizebox{\linewidth}{!}{
\begin{tabular}
{llrrrrrrr}
\toprule
\multirow{ 3}{*}{SVD}& \multicolumn{1}{l}{Model Size (Billion)} & 6.74 & 6.14 & 3.13 & 2.37 & 1.69 & 1.09 & \\\cmidrule{2-8}
 & \multicolumn{1}{l}{HumanEval Acc (\%)} & 45.1 & 37.8 & 27.4 & 14.0  & 9.8 & 3.0 & \\
 & \multicolumn{1}{l}{MBPP Acc (\%)} & 59.8 & 54.8 & 46.8 & 35.7 & 19.6 & 6.1 & \\\midrule
\multirow{ 3}{*}{FWSVD}& \multicolumn{1}{l}{Model Size (Billion)} & 6.74 & 4.77 & 2.94 & 2.19 & 1.54 & 0.96 & \\\cmidrule{2-8}
 & \multicolumn{1}{l}{HumanEval Acc (\%)} & 45.1 & 28.7 & 22.6 & 20.1 & 9.1 & 3.7 &  \\
 & \multicolumn{1}{l}{MBPP Acc (\%)} & 59.8 & 54.8 & 44.2 & 36.2 & 23.8 & 9.5 &  \\\midrule

\multirow{ 3}{*}{ASVD}& \multicolumn{1}{l}{Model Size (Billion)} & 6.74 & 4.10 & 3.43 & 2.77 & 2.11 & 1.45 & \\\cmidrule{2-8}
 & \multicolumn{1}{l}{HumanEval Acc (\%)} & 45.1 & 27.4 & 26.8 & 22.6 & 14.6 & 5.5 &  \\
 & \multicolumn{1}{l}{MBPP Acc (\%)} &59.8 & 51.9 & 44.7 & 43.9 & 32.3 & 16.1 &  \\\midrule

\multirow{ 3}{*}{AFM}& \multicolumn{1}{l}{Model Size (Billion)} & 6.74 & 3.76 & 2.81 & 2.04 & 1.42 & 0.92 & \\\cmidrule{2-8}
 & \multicolumn{1}{l}{HumanEval Acc (\%)} & 45.1 & 31.7 & 25.6 & 19.5 & 4.9 & 4.9 & \\
 & \multicolumn{1}{l}{MBPP Acc (\%)} & 59.8 & 49.5 & 49.2 & 43.1 & 26.7 & 12.2 &\\\midrule
 \multirow{ 3}{*}{IMPACT}& \multicolumn{1}{l}{Model Size (Billion)} & 6.74 & 3.58 & 2.64 & 1.90 & 1.30 & 0.84 & \\\cmidrule{2-8}
 & \multicolumn{1}{l}{HumanEval Acc (\%)} & 45.1 & 33.5 & 27.4 & 23.2 & 11.0 & 4.9 & \\
 & \multicolumn{1}{l}{MBPP Acc (\%)} & 59.8 & 51.6 & 49.7 & 42.6 & 28.0 & 16.4 & 
\\\bottomrule
\end{tabular}
}
\caption{Pass@1 accuracy and model size of CodeLlama-7B compressed with various low-rank algorithms on the code generation task. The results for SVD and FWSVD are taken from \citet{basel}.}
\label{tbl:7b-prog}
\end{table}



\begin{table}[!h]
\centering
\resizebox{\linewidth}{!}{
\begin{tabular}
{llrrrrrrr}
\toprule
\multirow{ 3}{*}{SVD}& \multicolumn{1}{l}{Model Size (Billion)} & 13.02 & 9.54 & 5.94 & 4.47 & 3.16 & 1.77 & \\\cmidrule{2-8}
 & \multicolumn{1}{l}{HumanEval Acc (\%)} & 52.4 & 36.6 & 28.0 & 18.3 & 14.6 & 1.8 & \\
 & \multicolumn{1}{l}{MBPP Acc (\%)} & 63.0 & 58.2 & 49.7 & 47.4 & 30.7 & 1.3 &\\\midrule
\multirow{ 3}{*}{FWSVD}& \multicolumn{1}{l}{Model Size (Billion)} & 13.02 & 9.17 & 5.29 & 4.14 & 2.30 & 1.76 & \\\cmidrule{2-8}
 & \multicolumn{1}{l}{HumanEval Acc (\%)} & 52.4 & 38.4 & 29.3 & 21.3 & 13.4 & 2.4 &   \\
 & \multicolumn{1}{l}{MBPP Acc (\%)} & 63.0 & 57.7 & 46.6 & 44.7 & 23.0 & 4.5 &   \\\midrule

\multirow{ 3}{*}{ASVD}& \multicolumn{1}{l}{Model Size (Billion)} & 13.02 & 9.16 & 5.32 & 4.01 & 2.73 & 1.45 & \\\cmidrule{2-8}
 & \multicolumn{1}{l}{HumanEval Acc (\%)} & 52.4 & 28.7 & 19.5 & 20.1 & 6.1 & 4.3&   \\
 & \multicolumn{1}{l}{MBPP Acc (\%)} & 63.0 & 53.4 & 44.2 & 38.9 & 29.4 & 10.3 &   \\\midrule

\multirow{ 3}{*}{AFM}& \multicolumn{1}{l}{Model Size (Billion)} & 13.02 & 9.68 & 5.45 & 3.96 & 2.74 & 1.75 & \\\cmidrule{2-8}
 & \multicolumn{1}{l}{HumanEval Acc (\%)} & 52.4 & 43.3 & 34.8 & 21.3 & 14.6 & 6.1 & \\
 & \multicolumn{1}{l}{MBPP Acc (\%)} & 63.0 & 59.8 & 51.3 & 43.4 & 31.2 & 19.0 &\\\midrule
 \multirow{ 3}{*}{IMPACT}& \multicolumn{1}{l}{Model Size (Billion)} & 13.02 & 9.61 & 5.16 & 3.71 & 2.44 & 1.66 & \\\cmidrule{2-8}
 & \multicolumn{1}{l}{HumanEval Acc (\%)} & 52.4 & 47.6 & 34.1 & 20.7 & 15.2 & 8.5 & \\
 & \multicolumn{1}{l}{MBPP Acc (\%)} & 63.0 & 61.6 & 51.3 & 45.5 & 31.7 & 25.9 &
\\\bottomrule
\end{tabular}
}
\caption{Pass@1 accuracy and model size of CodeLlama-13B compressed with various low-rank algorithms on the code generation task. The results for SVD and FWSVD are taken from \citet{basel}.}
\label{tbl:13b-prog}
\end{table}



\begin{table}[!h]
\centering
\resizebox{\linewidth}{!}{
\begin{tabular}
{llrrrrrrr}
\toprule
\multicolumn{8}{c}{GSM8K} \\\midrule
 \multirow{1}{*}{\shortstack[l]{Best Baseline\footnotemark[2] \\(with Same Acc.)}}& \multicolumn{1}{l}{Model Size (Billion)} & 6.74 & 5.76 & 5.12 & 3.75 & 1.69 & 1.02&\\
 \\\midrule
\multirow{2}{*}{IMPACT}& \multicolumn{1}{l}{Model Size (Billion)} & 6.74 & 3.11 & 2.40 & 1.67 & 1.19 & 0.75 &\\\cmidrule{2-8}
 & \multicolumn{1}{l}{Size Reduction (\%)} & - & 46.0 & 53.2 & 55.4 & 29.8 & 25.7 & \\\midrule
 \multicolumn{8}{c}{MATH} \\\midrule
 \multirow{1}{*}{\shortstack[l]{Best Baseline\footnotemark[2] \\(with Same Acc.)}}& \multicolumn{1}{l}{Model Size (Billion)} & 6.74 & 4.15 & 3.51 & 1.77 & 1.71 & 1.06 &\\
 \\\midrule
\multirow{2}{*}{IMPACT}& \multicolumn{1}{l}{Model Size (Billion)} & 6.74 & 3.11 & 2.40 & 1.67 & 1.19 & 0.75 &\\\cmidrule{2-8}
 & \multicolumn{1}{l}{Size Reduction (\%)} & - & 25.0 & 31.7 & 5.6 & 30.5 & 28.7  &

\\\bottomrule
\end{tabular}
}
\caption{Size reduction of IMPACT compared to the best baseline at matched accuracy for Llama~2-7B on the mathematical reasoning task.}
\label{tbl:7b-math-size-reduction}
\end{table}

\begin{table}[!ht]
\centering
\resizebox{\linewidth}{!}{
\begin{tabular}
{llrrrrrrr}
\toprule
\multicolumn{8}{c}{GSM8K} \\\midrule
 \multirow{1}{*}{\shortstack[l]{Best Baseline\footnotemark[2] \\(with Same Acc.)}}& \multicolumn{1}{l}{Model Size (Billion)} & 13.02 & 8.97 & 8.10 & 3.70 &  2.69 &  2.17&\\
 \\\midrule
\multirow{2}{*}{IMPACT}& \multicolumn{1}{l}{Model Size (Billion)} & 13.02 & 8.66 & 4.87 & 3.58 & 2.28 & 1.49 & \\\cmidrule{2-8}
 & \multicolumn{1}{l}{Size Reduction (\%)} & - & 3.5 & 39.8 & 3.3 & 15.0 & 31.2 & \\\midrule
 \multicolumn{8}{c}{MATH} \\\midrule
 \multirow{1}{*}{\shortstack[l]{Best Baseline\footnotemark[2] \\(with Same Acc.)}}& \multicolumn{1}{l}{Model Size (Billion)} & 13.02 & 10.33 & 6.91 & 5.44 & 3.59 & 2.25 &\\
 \\\midrule
\multirow{2}{*}{IMPACT}& \multicolumn{1}{l}{Model Size (Billion)} & 13.02 & 8.66 & 4.87 & 3.58 & 2.28 & 1.49 & \\\cmidrule{2-8}
 & \multicolumn{1}{l}{Size Reduction (\%)} & - & 16.2 & 29.4 & 34.2 & 36.5 & 33.7&

\\\bottomrule
\end{tabular}
}
\caption{Size reduction of IMPACT compared to the best baseline at matched accuracy for Llama~2-13B on the mathematical reasoning task.}
\label{tbl:13b-math-size-reduction}
\end{table}

\begin{table}[!h]
\centering
\resizebox{\linewidth}{!}{
\begin{tabular}
{llrrrrrrr}
\toprule
\multicolumn{7}{c}{HumanEval} \\\midrule
 \multirow{1}{*}{\shortstack[l]{Best Baseline\footnotemark[2] \\(with Same Acc.)}}& \multicolumn{1}{l}{Model Size (Billion)} & 6.74 & 4.16 & 3.09 & 2.51 & 1.65 & 1.09&\\
 \\\midrule
\multirow{2}{*}{IMPACT}& \multicolumn{1}{l}{Model Size (Billion)} & 6.74 & 3.58 & 2.64 & 1.90 & 1.30 & 0.84 & \\\cmidrule{2-8}
 & \multicolumn{1}{l}{Size Reduction (\%)} & -  & 13.8 & 14.4 & 24.2 & 21.0 & 23.1 & \\\midrule
 \multicolumn{7}{c}{MBPP} \\\midrule
 \multirow{1}{*}{\shortstack[l]{Best Baseline\footnotemark[2] \\(with Same Acc.)}}& \multicolumn{1}{l}{Model Size (Billion)} & 6.74 & 4.07 & 3.82 & 2.02 & 1.47 & 1.06&\\
 \\\midrule
\multirow{2}{*}{IMPACT}& \multicolumn{1}{l}{Model Size (Billion)} & 6.74 & 3.58 & 2.64 & 1.90 & 1.30 & 0.84 & \\\cmidrule{2-8}
 & \multicolumn{1}{l}{Size Reduction (\%)} & - & 12.0 & 30.7 & 6.1 & 11.1 & 21.1&

\\\bottomrule
\end{tabular}
}
\caption{Size reduction of IMPACT compared to the best baseline at matched accuracy for CodeLlama-7B on the code generation task.}
\label{tbl:7b-programming-size-reduction}
\end{table}

\begin{table}[!h]
\centering
\resizebox{\linewidth}{!}{
\begin{tabular}
{llrrrrrrr}
\toprule
\multicolumn{8}{c}{HumanEval} \\\midrule
 \multirow{1}{*}{\shortstack[l]{Best Baseline\footnotemark[2] \\(with Same Acc.)}}& \multicolumn{1}{l}{Model Size (Billion)} & 13.02 & 11.25 & 5.38 & 3.85 & 2.72 & 2.03 & \\
 \\\midrule
\multirow{2}{*}{IMPACT}& \multicolumn{1}{l}{Model Size (Billion)} & 13.02 & 9.61 & 5.16 & 3.71 & 2.44 & 1.66 & \\\cmidrule{2-8}
 & \multicolumn{1}{l}{Size Reduction (\%)} & - & 14.6 & 3.9 & 3.7 & 10.2 & 18.1 & \\\midrule
 \multicolumn{8}{c}{MBPP} \\\midrule
 \multirow{1}{*}{\shortstack[l]{Best Baseline\footnotemark[2] \\(with Same Acc.)}}& \multicolumn{1}{l}{Model Size (Billion)} & 13.02 & 11.55 & 5.45 & 4.32 & 2.79 & 2.31 & \\
 \\\midrule
\multirow{2}{*}{IMPACT}& \multicolumn{1}{l}{Model Size (Billion)} & 13.02 & 9.61 & 5.16 & 3.71 & 2.44 & 1.66 & \\\cmidrule{2-8}
 & \multicolumn{1}{l}{Size Reduction (\%)} & - & 16.9 & 5.3 & 14.2 & 12.5 & 28.0 & 

\\\bottomrule
\end{tabular}
}
\caption{Size reduction of IMPACT compared to the best baseline at matched accuracy for CodeLlama-13B on the code generation task.}
\label{tbl:13b-programming-size-reduction}
\end{table}

\footnotetext[2]{The best baseline refers to the smallest model among SVD, FWSVD, ASVD, and AFM that achieves accuracy matched to IMPACT. If no baseline exactly matches the accuracy, the model size is interpolated linearly between two adjacent compression points.}



\begin{table}[!t]
\centering
\resizebox{\linewidth}{!}{
\begin{tabular}{llrrrrrr}
\toprule

\multirow{ 3}{*}{SVD} 
& \multicolumn{1}{l}{Model Size (Billion)} 
& 6.74 & 5.03 & 3.18 & 1.73 & 1.11 & 0.56 \\\cmidrule{2-8}
& \multicolumn{1}{l}{Throughput (Token/s)} 
& 261.74 & 321.71 & 373.47 & 489.00 & 545.19 & 607.92 \\
& \multicolumn{1}{l}{Memory (GB)} 
& 29.26 & 22.73 & 15.34 & 9.75 & 6.90 & 4.70 \\\midrule

\multirow{ 3}{*}{FWSVD} 
& \multicolumn{1}{l}{Model Size (Billion)} 
& 6.74 & 4.79 & 2.95 & 1.54 & 0.96 & 0.47 \\\cmidrule{2-8}
& \multicolumn{1}{l}{Throughput (Token/s)} 
& 261.74 & 341.30 & 402.54 & 506.01 & 553.89 & 639.70 \\
& \multicolumn{1}{l}{Memory (GB)} 
& 29.26 & 21.93 & 14.56 & 9.04 & 6.62 & 4.38 \\\midrule

\multirow{ 3}{*}{AFM} 
& \multicolumn{1}{l}{Model Size (Billion)} 
& 6.74 & 3.62 & 2.66 & 1.90 & 1.30 & 0.83 \\\cmidrule{2-8}
& \multicolumn{1}{l}{Throughput (Token/s)} 
& 261.74 & 344.59 & 414.04 & 470.98 & 538.42 & 566.43 \\
& \multicolumn{1}{l}{Memory (GB)} 
& 29.26 & 17.03 & 13.24 & 10.21 & 7.70 & 5.76 \\\midrule

\multirow{ 3}{*}{IMPACT} 
& \multicolumn{1}{l}{Model Size (Billion)} 
& 6.74 & 3.11 & 2.40 & 1.67 & 1.19 & 0.75 \\\cmidrule{2-8}
& \multicolumn{1}{l}{Throughput (Token/s)} 
& 261.74 & 414.03 & 445.16 & 505.74 & 568.04 & 616.01 \\
& \multicolumn{1}{l}{Memory (GB)} 
& 29.26 & 14.99 & 12.21 & 9.25 & 7.25 & 5.41 \\\bottomrule

\end{tabular}
}
\caption{Throughput and memory consumption of compressed models. The results for SVD and FWSVD are taken from \citet{basel}.}
\label{tbl:inference-results}
\end{table}

\end{document}